\documentclass{article}

\usepackage[utf8]{inputenc} 
\usepackage[T1]{fontenc}    
\usepackage{hyperref}       
\usepackage{url}            
\usepackage{booktabs}       
\usepackage{amsfonts}       
\usepackage{nicefrac}       
\usepackage{microtype}      
\usepackage{lipsum}

\usepackage{amsmath}
\usepackage{amssymb}
\usepackage{amsthm}
\usepackage{algorithm}
\usepackage{algorithmic}
\usepackage{fullpage}
\usepackage{tikz}
\usepackage{tabu}
\usepackage{graphics}
\usepackage{subfig}
\usepackage{epstopdf}
\usepackage{epsfig}
\usepackage{caption}
\usepackage{float}
\newtheorem{theorem}{Theorem}
\newtheorem{lemma}[theorem]{Lemma}         
\newtheorem{corollary}[theorem]{Corollary}     

\newtheorem{definition}[theorem]{Definition}

\def\R{{\mathbb{R}}}
\def\E{{\mathbb{E}}}

\def\rank{\text{rank}}
\def\O{{\mathcal{O}}}
\def\P{{\mathcal{P}}}
\def\pr{\mathbb{P}}

\def\A{\mathcal{A}}
\def\T{\widetilde}
\def\Ta{{\mathcal{T}}}
\def\H{{\mathcal{H}}}
\def\M{\mathcal{M}}
\def\U{\mathbb{U}}
\def\n{\Big{\|}}
\usepackage{enumitem}

\usepackage{todonotes}
\def\poly{\text{poly}}
\def\Tr{\text{Tr}}

\title{Improved Algorithms for Matrix Recovery \\ from Rank-One Projections 
}

\author{
  Mohammadreza Soltani and Chinmay Hegde \\
  Electrical and Computer Engineering Department\\
  Iowa State University\\
}
\date{\vspace{-1ex}}
\begin{document}

\maketitle
\begin{abstract}
We consider the problem of estimation of a low-rank matrix from a limited number of noisy rank-one projections. In particular, we propose two fast, non-convex \emph{proper} algorithms for matrix recovery and support them with rigorous theoretical analysis. We show that the proposed algorithms enjoy linear convergence and that their sample complexity is independent of the condition number of the unknown true low-rank matrix. By leveraging recent advances in low-rank matrix approximation techniques, we show that our algorithms achieve computational speed-ups over existing methods. Finally, we complement our theory with some numerical experiments.
\end{abstract}

\section{Introduction}
\label{sec:intro}

\subsection{Setup}

This paper studies the following inverse problem: given a fixed (but unknown) rank-$r$ symmetric matrix $L_* \in \R^{p \times p}$, recover $L_*$ from a small number of \emph{rank-one} projections of the form:
\begin{equation}
\label{eq:model}
y_i = x_i^T L_* x_i + e_i = \langle x_i x_i^T, L_* \rangle + e_i,~~~i = 1,\ldots,m.
\end{equation}
where $x_i \in \R^p$ denote (known) feature vectors, and $e_i$ denotes stochastic noise. 
This inverse problem can be used to model numerous challenges in statistics and machine learning, including:

\noindent \textbf{Matrix sensing.} Reconstructing low-rank matrices from (noisy) linear measurements of the form $y_i = \langle X_i, L_*\rangle$ impacts several applications in control and system identification~\cite{fazel2002matrix}, collaborative filtering~\cite{recht2010guaranteed}, and imaging. The problem~\eqref{eq:model} specializes the matrix sensing problem to the case where the measurement vectors $X_i$ are constrained to be the specific rank-one form; recently, such measurements has been shown to provide several computational benefits \cite{zhong2015efficient,kueng2017low}.

\noindent \textbf{Covariance sketching.} Estimating a high-dimensional covariance matrix, given a stream of independent samples $\{s_t\}_{t=1}^\infty,~s_t \in \R^p$, involves maintaining the empirical estimate $Q = E[s_t s_t^T]$, which can require \emph{quadratic} ($O(p^2)$) space complexity. Alternatively, one can record a sequence of $m \ll p^2$ \emph{linear sketches} of each sample: $z_i = x_i^T s_t$ for $i = 1,\ldots,m$. At the conclusion of the stream, sketches corresponding to a given vector $x_i$ are squared and aggregated to form a measurement: $y_i = E[z_i^2] = E[ (x_i^T s_t)^2] x_i = x_i^T Q x_i$, which is nothing but a sketch of $Q$ in the form \eqref{eq:model}; efficient recovery methods to invert such sketches exist~\cite{cai2015rop,chen2015exact} .

\noindent \textbf{Polynomial neural network learning.} Consider a shallow (two-layer) neural network architecture with $p$ input nodes, a single hidden layer with $r$ neurons with \emph{quadratic} activation function $\sigma(z)= z^2$ and weights $\{w_j\}_{j=1}^r \subset \R^p$, and an output layer comprising of a single node and weights $\{\alpha_j\}_{j=1}^r \subset \R$. The input-output relationship between an input $x \in \R^p$ and the corresponding output $y$ can be described as~\cite{livni2014computational}:
\[
y = \sum_{j=1}^r \alpha_j \langle w_j, x \rangle^2 = x^T \left( \sum_{j=1}^r \alpha_j w_j w_j^T \right) x.
\] 
Here, the learning problem is to estimate the weights $\{\alpha_j,w_j\}$ as accurately as possible given a sequence of training input-output pairs $\{(x_i,y_i)\}_{i=1}^m$. The recent works~\cite{livni2014computational,lin2016non} explore efficient algorithms to recover the weights of such a ``ground-truth" network under distributional assumptions. 
 
\subsection{Our contributions}

In this paper, we make concrete algorithmic progress on solving problems of the form~\eqref{eq:model}.  

A range of algorithms for solving~\eqref{eq:model} (or variants thereof) exist in the literature, and can be broadly classified into two categories: (i) \emph{convex} approaches, all of which involve modeling the rank-$r$ assumption in terms of a convex penalty term, such as the nuclear norm~\cite{fazel2002matrix,recht2010guaranteed,cai2015rop,chen2015exact}, and (ii) \emph{nonconvex} approaches based on either alternating minimization~\cite{zhong2015efficient,lin2016non} or greedy approximation~\cite{livni2014computational,shalev2011large}. Both types of approaches suffer from severe computational difficulties, particularly in the high dimensional regime. Even the most computationally efficient convex approaches inevitably require multiple invocations of singular value decomposition (SVD) of a potentially large $p \times p$ matrix, which can incur \emph{cubic} ($O(p^3)$) running time. Moreover, even the best available non-convex approaches inevitably require a very accurate initialization, and that the underlying matrix $L_*$ is well-conditioned (if this is not the case, the running time of all available methods again inflates to $O(p^3)$, or worse). 

We take a different approach, and show how to leverage recent results in randomized numerical analysis~\cite{musco2015randomized,HegdeFastUnionNips2016} to our advantage. Our algorithm is also non-convex; however, unlike all earlier works, our method does not require any full SVD calculations. Specifically, we demonstrate that a careful concatenation of \emph{randomized, approximate} SVD methods, coupled with appropriately defined gradient steps, leads to  efficient and accurate estimation. To our knowledge, this work constitutes the first method for matrix recovery for rank-one projections that has \emph{nearly linear} running time, is \emph{nearly sample-optimal} for fixed target rank $r$, and is \emph{unconditional} (i.e., it does not depend on the condition number of the underlying matrix.) Numerical experiments reveal that our methods yield a very attractive tradeoff between sample complexity and running time for efficient matrix recovery.

\begin{table}[!t]
 \caption{Summary of the contribution and comparison with existing algorithms. Here, SD stands for Spectral Dependency, i.e., the running time depends on $\sqrt{\frac{\sigma_r}{\sigma_{r+1}}-1}$ which is a spectral gap parameter. Also, $\beta= \frac{\sigma_1}{\sigma_{r}}$ denotes the matrix condition number.}
  \label{sample-table}
  \vskip .01 in 
  \centering
  \begin{tabular}{llll}
    \toprule
    \cmidrule{1-2}
    Algorithm &Sample complexity & Total Running Time & SD \\
        \midrule
Convex \cite{chen2015exact,cai2015rop,kueng2017low} & $\O(pr)$ &$\O\left(\frac{p^3}{\epsilon}\right)$ & Yes\\ 
GECO \cite{livni2014computational,shalev2011large} & N/A &{$\O\left(\frac{p^2\log(p)\poly(r)}{\epsilon}\right) $} &Yes \\
AltMin-LRROM \cite{zhong2015efficient} & $\O\left(pr^4\log^2(p)\beta^2\log(\frac{1}{\epsilon})\right)$ &  $\O\left(p^2r^5\log^2(p)\beta^2\log^2(\frac{1}{\epsilon}) + p^3\right)$ & Yes\\ 
gFM \cite{lin2016non} &$\O(pr^3\beta^2\log(\frac{1}{\epsilon}))$ & $\O\left(p^2r^4\beta^2\log^2(\frac{1}{\epsilon})+p^3\right)$ & Yes \\
EP-ROM \textbf{[This paper]} & $\O\left(pr^2\log^4(p)\log(\frac{1}{\epsilon})\right)$ & $\O\left(p^3r^2\log^4(p)\log^2(\frac{1}{\epsilon})\right)$ & \textbf{Yes}\\ 
AP-ROM \textbf{[This paper]} & $\O\left(pr^2\log^4(p)\log(\frac{1}{\epsilon})\right)$ & $\O\left(p^2 r^3\log^5(p)\log^2(\frac{1}{\epsilon})\right)$  &  \textbf{No} \\
    \bottomrule
  \end{tabular}
      \vskip -0.15 in
\end{table}

\subsection{Techniques}

At a high level, our method can be viewed as a variant of the seminal algorithms proposed in~\cite{jain2010guaranteed,cai2010singular}, which essentially constitute projected/proximal gradient descent with respect to the space of rank-$r$ matrices. However, since computing SVD in high dimensions can be a bottleneck, we cannot use this approach directly. To this end, we use the \emph{approximation-}based matrix recovery framework proposed in~\cite{HegdeFastUnionNips2016}. This work demonstrates how to carefully integrate approximate SVD methods into SVP-based matrix recovery algorithms; in particular, algorithms that satisfy certain ``head" and "tail" projection properties (explained below in Section~\ref{sec:model}) are sufficient to guarantee robust and fast convergence. This framework enables us to remove the need to compute even a single SVD.

However, a direct application of~\cite{HegdeFastUnionNips2016} does not succeed for the observation model~\eqref{eq:model}. Two major obstacles arise in our case: (i) It is known the measurement operator that maps $L_*$ to $y$ does not satisfy the Restricted Isometry Property (in terms of the Frobenius norm) over rank-$r$ matrices~\cite{cai2015rop,chen2015exact,zhong2015efficient}; therefore, all statistical and algorithmic correctness arguments no longer apply. (ii) The structure of the rank-one measurement inflates the running time of computing even a simple gradient update to $O(p^3)$ (ignoring any cost incurred during rank-$r$ projection, whether done using exact or approximate SVDs).

We resolve the first issue by studying the concentration properties of certain linear operators based on the rank-one projections, leveraging an approach first proposed in~\cite{zhong2015efficient}. We show that a non-trivial ``bias correction" step to projected descent-type methods together with a fresh
set of samples in each iteration is necessary to achieve fast (linear) convergence. To be more precise, let the operator $\A$ be defined as $(\A(L_*))_i = x_i^T L _*x_i$ for $i=1,\ldots,m$, where $x_i$ is a standard normal random vector. A simple calculation shows that at any given iteration $t$, if $L_t$ is the estimate of the underlying low rank matrix then:
$$ \E\A^*\A(L_t - L_*) = 2(L_t - L_*) + Tr(L_t - L_*)I,$$
where the operator $Tr(\cdot)$ denotes the trace of a matrix and the expectation is taken with respect to the randomness in the $x_i$'s. The left hand side of this equation (roughly) corresponds to the expected value of the gradient in each iteration, and it is clear that while the gradient points in the ``correct" direction $L_t - L_*$, it is additionally biased by the extra $Tr(.)$ term. Motivated by this, we develop a new descent scheme that achieves linear convergence by carefully accounting for this bias. Interestingly, the sample complexity of our approach only increases by a mild factor (specifically, an extra factor $r$ together with polylogarithmic terms) when compared to the best available techniques even by using fresh samples in each iteration which increases the sample complexity by a log factor due to the linear convergence of the proposed algorithm. 

We resolve this issue by carefully exploiting the rank-one structure of the observations. In particular, we develop a modification of the randomized block-Krylov SVD (BK-SVD) algorithm of~\cite{musco2015randomized} to work for the case of certain ``implicitly defined" matrices, i.e., where the input to the SVD routine is not a matrix, but rather an linear operator that is constructed using the current estimate .  This modification, coupled with the tail- and head-projection arguments developed in~\cite{HegdeFastUnionNips2016} enables us to achieve a fast per-iteration computational complexity. In particular, our algorithm strictly improves over the (worst-case) per-iteration running time of all existing algorithms.

\section{Related Work}
\label{sec:prior}

Due to space constraints, we only provide here a brief (and necessarily incomplete) review of related work.
The low-rank matrix recovery problem has received significant attention from the machine learning community over the last few years~\cite{davenport2016overview}. The so-called \emph{affine rank minimization} problem can be thought as the generalization of popular compressive sensing~\cite{candes2008introduction,foucart2013} to the set of low-rank matrices and has several applications in signal processing, communications, control, and computer vision~\cite{fazel2002matrix,recht2010guaranteed,jung2013compressive}. In early works for matrix recovery, the observation operator $\A$ is assumed to be parametrized by $m$ independent \emph{full-rank} $p\times p$ matrices that satisfy certain \emph{restricted isometry} conditions~\cite{recht2010guaranteed,liu2011universal}. In this setup, it has been established that $m = \O(pr)$ observations are sufficient to recover an unknown rank-$r$ matrix $L_*$ in~\eqref{Obmodel}~\cite{candes2011tight}, which is statistically optimal. 

More recently, several works have studied the special case for which the operator $\A$ is parametrized by $m$ independent rank-$1$ matrices, as introduced above in~\eqref{eq:model}~\cite{chen2015exact,cai2015rop,kueng2017low,zhong2015efficient,Secondorder17}. This model arises in several key applications. Three popular practical applications include: (i) \emph{Matrix completion}~\cite{candes2009exact}. This problem is popular for studying collaborative filtering and recommender systems, and here every observation corresponds to a single entry of  $L_*$. Therefore, the rank-one observation matrices are of the form $\{x_i x_j^T\}$, with $x_i,x_j$ corresponding to canonical basis vectors. (ii) \emph{Covariance sketching} for real-time data monitoring \cite{chen2015exact,dasarathy2015sketching}. Here, a stream of high dimensional data $\{x_t\} \subset \R^p$ is sequentially observed at a high rate, and acquisition systems with memory and processing limitations can only record a few summaries (or \emph{sketches} of individual data points in the form $\{a_j^T x\}$. The eventual goal would be to infer the second order statistics (i.e., covariance matrix) from the recorded sketches. 
(iii) \emph{Phase retrieval}. This is a well-known challenge in certain types of imaging systems where the goal is to reconstruct a signal (or image) from just amplitude information~\cite{candes2013phaselift,netrapalli2013phase}. 

We now briefly discuss algorithmic techniques for low-rank matrix recovery from rank-one projections. The first class of techniques can be categorized as convex-relaxation based approaches~\cite{chen2015exact,cai2015rop,kueng2017low,candes2013phaselift}. For instance, the authors in~\cite{chen2015exact,cai2015rop} demonstrate that the observation operator $\A$ satisfies a specialized mixed-norm isometry condition called the \emph{RIP-$\ell_2/\ell_1$}. Further, they show that the sample complexity of matrix recovery using rank-one projections matches the optimal rate $\O(pr)$. However, these methods advocate using either semidefinite programming (SDP) or proximal sub-gradient algorithms~\cite{boyd2004convex}, both of which are too slow for very high-dimensional problems. 

The second class of techniques include non-convex methods which are all based on a factorization-based approach initially advocated by Burer and Monteiro. Here, the underlying low-rank matrix is factorized as $L_* = UV^T$, where $U,V\in\R^{p\times r}$. In the {Altmin-LRROM} method by~\cite{zhong2015efficient}, $U$ and $V$ are updated in alternative fashion such that in each iteration the columns of $U$ and $V$ are orthonormalized, 
while in the generalized Factorization Machine (gFM) method by~\cite{Secondorder17}, $U$ and $V$ are updated based on the construction of certain sequences of moment estimators. 
However both of these approaches require a spectral initialization which involves running a rank-$r$ SVD on a  given $p\times p$ matrix, and therefore the running time heavily depends on the condition number (i.e., the ratio of the maximum and the minimum nonzero singular values) of $L_*$. 
Our proposed approach is also non-convex, and parts of our statistical analysis is based on certain matrix concentration results introduced in~\cite{zhong2015efficient,Secondorder17}. However, unlike these previous works, our method does not get adversely affected by poor matrix condition number either in theory (as demonstrated in~Table~\ref{sample-table}) or in practice (as demonstrated by our numerical experiments below). We achieve this improvement by leveraging new methods for randomized approximate SVD~\cite{musco2015randomized} coupled with the framework of~\cite{HegdeFastUnionNips2016}.

Finally, we mention that a related matrix recovery scheme using approximate SVDs (based on Frank-Wolfe type greedy approximation) has been proposed for learning polynomial neural networks~\cite{shalev2011large,livni2014computational}. Moreover, this approach has been shown to compare favorably to typical neural network learning methods (such as stochastic gradient descent); however, the rate of convergence is sub-linear. We extend this line of work by providing (a) rigorous statistical analysis that precisely establishes upper bounds on the number of samples required for learning such networks, and (b) an algorithm that exhibits \emph{linear} convergence to the desired solution; see Table~\ref{sample-table}.


\section{Main Results}
\label{sec:model}

\subsection{Preliminaries}
Let us first introduce some notations. Throughout this paper, $\| \cdot \|_F$ and $\| \cdot \|_2$ denote the Frobenius and spectral norm, respectively, and $\Tr(\cdot)$ denotes the trace  of a matrix. For any given set $\mathcal{U} \in \R^{p\times p}$ we denote the orthogonal projection onto $\mathcal{U}$ by $\P_\mathcal{U}$.
The underlying observation model~\eqref{eq:model} can be represented as follows:
\begin{align}
\label{Obmodel}
y = \A(L_*) + e 
\end{align}
where $L_*\in\R^{p\times p}$ is a symmetric matrix, the linear operator, $\A:\R^{p\times p}\rightarrow\R^m$ defined as $\A(L_*) = [x_1^TL_*x_1,x_2^TL_*x_2,\ldots,x_m^TL_*x_m]^T$ and each $x_i\overset{i.i.d}{\thicksim}\mathcal{N}(0,I)$ is a random vector in $\R^p$ for $i=1,\ldots,m$. 
Also, the adjoint operator, $\A:\R^m\rightarrow\R^{p\times p}$ is defined as $\A^*(y) = \sum_{i=1}^{m}y_ix_ix_i^T$. Here, $e\in\R^m$ denotes additive noise, and throughout the paper we assume that $e$ is zero-mean, subgaussian with i.i.d entries, and independent of $x_i$'s. 
The goal is to recover the unknown low-rank matrix $L^*$ with $\rank$ $r\ll p$ from as few observations as possible.

In the analysis of our algorithms, we need the following regularity condition of the operators $\A$ and $\A^*$ with respect to the set of low-rank matrices. which we term as the \textit{Conditional Unbiased Restricted Isometry Property}, abbreviated as \textit{CU-RIP$(\rho_1,\rho_2,\rho_3)$}:
\begin{definition}(CU-RIP)\label{curip}
Consider fixed rank-$r$ matrices $L_1$ and $L_2$. {Assume that $y = \A(L_2)$} and $\bar{y} = \frac{1}{m}\sum_{i=1}^{m}y_i$ for linear operator $\A:\R^{p\times p}\rightarrow\R^m$ parametrized by $m$ rank 1 matrices. Then $\A$ is said to be satisfied CU-RIP$(\rho_1,\rho_2,\rho_3)$ 
if there exits $\rho_1>0$, $0<\rho_2<1$ and arbitrary small constant $\rho_3>0$ such that 
\begin{align}
\Big{\|}L_1-L_2 - \rho_1\left(\frac{1}{m}\A^*\A(L_1-L_2)- Tr(L_1 - \bar{y})I \right)\Big{\|}_2\leq \rho_2\Big{\|}L_1 - L_2\Big{\|}_2 + \rho_3Tr(L_2).
\end{align}  
\end{definition}

Let $\U_r$ denote the set of all rank-$r$ matrix subspaces, i.e., subspaces of $\mathbb{R}^{p \times p}$ which are spanned by any $r$ atoms of the form $uv^T$ where $u, v \in \mathbb{R}^p$ are unit $\ell_2$-norm vectors. We use the idea of \emph{head} and \emph{tail} approximate projections with respect to $\U_r$ first proposed in~\cite{hegde2015approximation}, and instantiated in the context of low-rank approximation in~\cite{HegdeFastUnionNips2016}.
\begin{definition}[Approximate tail projection]
\label{taildef}
Let $c_{\Ta}>1$. Then $\Ta:\R^{p\times p}\rightarrow\U_{r}$ is a $c_{\Ta}$-approximate tail projection algorithm if for all $L\in^{p\times p}$, $\Ta$ returns a subspace $W=\Ta(L)$ that satisfies: $\|L -\P_W L\|_F\leq c_{\Ta}\|L-L_r\|_F$, where $L_r$ is the optimal rank-$r$ approximation of $L$.
\end{definition}
\begin{definition}[Approximate head projection]
\label{headdef}
Let $0<c_{\H}< 1$ be a constant. Then $\H:\R^{p\times p}\rightarrow\U_{r}$ is a $c_{\H}$-approximate head projection if for all $L\in^{p\times p}$, the returned subspace $V=\H(L)$ satisfies: $\|\P_VL\|_F\geq c_{\H}\|L_r\|_F$, where $L_r$ is the optimal rank-$r$ approximation of $L$.
\end{definition}
Observe that $c_\Ta = 1$ (resp. $c_\H = 1$), then the algorithms $\Ta$ and $\H$ can be realized using ordinary SVD techniques.

\subsection{Algorithms and theoretical results}

We now propose methods to estimate $L_*$ given knowledge of $\{x_i,y_i\}_{i=1}^m$. Our first method is somewhat computationally inefficient, but achieves very good sample complexity and serves to illustrate the overall algorithmic approach.
Consider the \emph{non-convex} optimization problem:
\begin{equation} \label{opt_prob}
\begin{aligned}
& \underset{L \in \R^{p \times p}}{\text{min}}
& & F(L) = \frac{1}{2m}\sum_{i=1}^{m}\left(y_i - x_i^TLx_i\right)^2
& \text{s.t.} 
& & \rank(L)\leq r.
\end{aligned}
\end{equation}
To solve this problem, we first propose an algorithm that we call \textit{Exact Projections for Rank-One Matrix} recovery, or \textit{EP-ROM}, described in pseudocode form in Algorithm~\ref{alg:exaMS}. In Alg~\ref{alg:exaMS}, $\P_r$ denotes the projection operator onto the set of rank-$r$ matrices. 
\begin{algorithm}[!t]
\caption{EP-ROM}
\label{alg:exaMS}
\begin{algorithmic}
\STATE \textbf{Inputs:} $y$, number of iterations K, a set of independent measurement operators $\{x_1^t,x_2^t\ldots,x_m^t\}$ for $t=1,\ldots,K$, rank $r$, step size $\eta$
\STATE \textbf{Outputs:} Estimates  $\widehat{L}$
\STATE\textbf{Initialization:} $L_0\leftarrow\textsc{zero initialization}$, $t \leftarrow 0$
\STATE\textbf{Calculate:} $\bar{y} = \frac{1}{m}\sum_{i=1}^{m}y_i$
\WHILE{$t\leq K$}
\STATE $L_{t+1} = \mathcal{P}_r\left(L_{t} - \eta\left(\frac{1}{m}\sum_{i=1}^{m}\left( (x_i^t)^TL_tx_i^t-y_i\right) x_i^t(x_i^t)^T-\left(Tr(L_t)-\bar{y}\right)I\right)\right)$
\STATE $t\leftarrow t+1$
\ENDWHILE
\STATE\textbf{Return:} $\widehat{L} = L_{K}$
\end{algorithmic}
\vskip -0.05in
\end{algorithm}

We now analyze this algorithm. First, we provide a theoretical result which establishes statistical and optimization convergence rates of EP-ROM. More precisely, we derive the upper bound on the estimation error (measured using the spectral norm) of recovering the unknown low-rank matrix $L_*$. Due to space constraints, we defer all the proofs to the appendix. 

\begin{theorem}[Linear convergence of EP-ROM]\label{LinConEX}
Consider the sequence of iterates $(L_t)$ obtained in EP-ROM. Let $V^t, V^{t+1}$, and $V^{*}$ denote the column spaces of $L^t, L^{t+1}$, and $L^*$, respectively. Define the subspace $J_t$ such that $V^t\cup V^{t+1}\cup V^* = J_t$. Also, assume that the linear operator $\A$ satisfies CU-RIP$(\rho_1,\rho_2,\rho_3)$ with parameters $\rho_1 = \eta$, $\rho_2 = \eta\delta_1 - 1$ and $\rho_3 = \eta\delta_2$ where $\eta$ is the step size of EP-ROM and $\delta_1, \delta_2>0$ are arbitrary small constants (to be specified later). Choose the step size in EP-ROM as $\frac{1}{\delta_1}<\eta <\frac{1.5}{\delta_1}$. Then, EP-ROM outputs a sequence of estimates $L_t$ such that:    
\begin{align}\label{lineconveqEx}
\|L_{t+1} - L_*\|_2 \leq q_1\|L_t-L_*\|_2  {+ q_2Tr(L_*)} + \frac{2\eta}{m}\|\P_{J_t}\A^*e\|_2,
\end{align}
where $q_1 = 2(\eta\delta_1 -1)$ and $q_2 = 2\eta\delta_2$.
\end{theorem}
Theorem~\ref{lineconveqEx} implies (via induction) that EP-ROM exhibits \emph{linear} convergence; further, the radius of convergence is dependent on the third term on the 
the third term on the right hand ~\eqref{lineconveqEx}. This term represents the \emph{statistical error rate}. We now prove that this error term can be suitably bounded. 
\begin{theorem}[Bounding the statistical error of EP-ROM]\label{Staterror}
Consider the observation model~\eqref{Obmodel} with zero-mean subgaussian noise $e\in\R^m$ with i.i.d. entries (and independent of the $x_i$'s) such that $\tau = \max_{1\leq j\leq m}\|e_j\|_{\psi_2}$ (Here, $\| \cdot \|_{\psi_2}$ denotes the $\psi_2$-norm; see Definition~\ref{subexp} in the appendix). Then, with probability at least $1-\xi_3$, we have:
\begin{align}
\Big{\|}\frac{1}{m}\A^*e\Big{\|}_2\leq C''_{\tau}\sqrt{\frac{p\log p}{m}\log(\frac{p}{\xi_3})}.
\end{align}
where $ C''_{\tau}>0$ is constant which depends on $\tau$.
\end{theorem}

In establishing linear convergence of EP-ROM, we assume that the CU-RIP holds at each iteration. The following theorem certifies this assumption by showing that this condition holds with high probability.
\begin{theorem}[Verifying CU-RIP$(\rho_1,\rho_2,\rho_3)$]\label{CURIP}
At any iteration $t$, with probability at least $1-\xi$, CU-RIP$(\rho_1,\rho_2,\rho_3)$ is satisfied with $\frac{1}{\delta_1}<\rho_1<\frac{2}{\delta_1}$, $\rho_2 = \rho_1\delta_1 - 1$, and $\rho_3 = \rho_1\delta_2$ where $\delta_1, \delta_2>0$ are some arbitrary small constants provided that $m=\O\left(\frac{1}{\delta^2}\max\left(pr^2\log^3p\log(\frac{p}{\xi}),\ {r^2\log(\frac{p}{\xi})\|L_2\|_2^2}\right)\right)$ for some $\delta>0$. 
\end{theorem}

Integrating the above results and using the idea of fresh sampling in each iteration, we obtain the following corollary formally establishing linear convergence. 
\begin{corollary}\label{Epinduc}
Consider all the assumptions in theorem~\ref{LinConEX}. To achieve $\epsilon$ accuracy in the estimation of $L_*$ in spectral norm, EP-ROM needs $K = \O(\log(\frac{\|L_*\|_2}{\epsilon}))$ iterations. In other words, the output of EP-ROM satisfies the following after $K$ iterations with high probability:
\begin{align}
\|L_{t+1} - L_*\|_2 \leq (q_1)^K\|L_*\|_2  + \frac{q_2}{1-q_1}Tr(L_*) + \frac{C''_{\tau}\eta}{1-q_1}\sqrt{\frac{p\log p}{m}\log(\frac{p}{\xi_3})}.
\end{align}
\end{corollary}

Based on Theorems~\ref{Staterror},~\ref{CURIP}, and Corollary~\ref{Epinduc}, the sample complexity of EPROM scales as \\ $m=\O\left(\frac{1}{\delta^2}\max\left(pr^2\log^3p\log(\frac{p}{\xi}),\ r^2\log(\frac{p}{\xi})\|L_*\|_2^2\right)\log(\frac{1}{\epsilon})\right)$ for some $\delta>0$.\\

While EP-ROM exhibits linear convergence, the \emph{per-iteration} complexity is still high since it requires projection onto the space of rank-$r$ matrices, which necessitates the application of SVD. In the absence of any spectral assumptions on the input to the SVD, the per-iteration running time of EP-ROM can be \emph{cubic}, which can be prohibitive. Overall, we obtain a running time of $\widetilde{\O}(p^3 r^2)$ in order to achieve $\varepsilon$-accuracy (please see section~\ref{TIMEAnalysis} in appendix for more discussion).
To reduce the running time, one can instead replace a standard SVD routine with heuristics such as Lanczos iterations~\cite{lanczos1950iteration}; however, these may may not result in algorithms with provable convergence guarantees. Instead, following~\cite{HegdeFastUnionNips2016}, we can use two inaccurate rank-$r$ projections (in particular, tail- and head-approximate projection operators), and we show that this leads to provable convergence. Based on this idea, we propose our second algorithm that we call \textit{Approximate Projection for Rank One Matrix} recovery, or \textit{AP-ROM}. We display the pseudocode of AP-ROM in Algorithm~\ref{alg:appMS}.

\begin{algorithm}[!t]
\caption{AP-ROM}
\label{alg:appMS}
\begin{algorithmic}
\STATE \textbf{Inputs:} $y$, number of iterations K, a set of independent measurement operators $\{x_1^t,x_2^t\ldots,x_m^t\}$ for $t=1,\ldots,K$, rank $r$, step size $\eta$
\STATE \textbf{Outputs:} Estimates  $\widehat{L}$
\STATE\textbf{Initialization:} $L_0\leftarrow\textsc{zero initialization}$, $t \leftarrow 0$
\STATE\textbf{Calculate:} $\bar{y} = \frac{1}{m}\sum_{i=1}^{m}y_i$
\WHILE{$t\leq K$}
\STATE $L_{t+1} = \Ta\left(L_{t} - \eta\H\left(\frac{1}{m}\sum_{i=1}^{m}\left( (x_i^t)^TL_tx_i^t-y_i\right) x_i^t(x_i^t)^T-\left(Tr(L_t)-\bar{y}\right)I\right)\right)$
\STATE $t\leftarrow t+1$
\ENDWHILE
\STATE\textbf{Return:} $\widehat{L} = L_{K}$
\end{algorithmic}
\vskip -0.05in
\end{algorithm}

The choice of approximate low-rank projections operators $\Ta(.)$ and $\H(.)$ is flexible. We note that tail approximate projections has been widely studied in the randomized numerical linear algebra literature~\cite{clarksonwoodruff_old,drineas_mahoney,tygert}; however, head approximate projection methods are less well-known. In our method, we use the randomized Block Krylov SVD (BK-SVD) method proposed by~\cite{musco2015randomized}, which has been shown to satisfy both types of approximation guarantees~\cite{HegdeFastUnionNips2016}. The nice feature of this method is that the running time incurred while computing a low-rank approximation of a given matrix is \emph{independent of the spectral gap} of the matrix. We leverage this property to show asymptotic improvements over other existing fast SVD methods (such as the well-known power method).

We briefly discuss the BK-SVD algorithm. In particular, BK-SVD takes an input matrix with size $p\times p$ with rank $r$ and returns a $r$-dimensional subspace which approximates the top right $r$ singular vectors of the input. Mathematically, if $A\in\R^{p\times p}$ is the input, $A_r$ is the best rank-$r$ approximation to it, and $Z$ is a basis matrix that spans the subspace returned by BK-SVD, then the projection of $A$ into $Z$, $B=ZZ^TA$ satisfies the following relations:
\begin{align*}
&\hspace{9mm} \|A-B\|_F\leq c_{\Ta}\|A-A_r\|_F,  \\
&|u_i^TAA^Tu_i - z_iAA^Tz_i|\leq (1-c_{\H})\sigma_{r+1}^2,
\end{align*}
where $c_{\Ta}>1$, $c_{\H}<1$, are the tail and head projection constants and $u_i$ denotes the $i^{th}$ right eigenvector of $A$. In section Appendix-B of~\cite{HegdeFastUnionNips2016} has been shown that the per-vector guarantee can be used to prove the approximate head projection property, i.e., $\|B\|_F\geq c_{\H}\|A_r\|_F$.
We now establish that AP-ROM also exhibits linear convergence, while obeying similar statistical properties as EP-ROM. We have the following results:
\begin{theorem}[Convergence of AP-ROM]\label{LinConAp}
Let $V_t = \H\left(\A^*(\A(L_t) -y) -Tr(L_t-\bar{y})I\right)$. Also, Assume that $\A$ satisfies CU-RIP$(\rho_1,\rho_2,\rho_3)$ with $\rho_1 = \eta$, $\rho_2 = \eta\delta_1 - 1$ and $\rho_3 = \eta\delta_2$ where $\eta$ is the step size of AP-ROM and $1<\delta_1<2$ and $\delta_2<2$ are arbitrary small positive constants. 
Choose the step size $\eta$ such that $\frac{1}{\delta_1} - \frac{\sqrt{1-\phi^2}}{\beta_1\delta_1\sqrt{r}}<\eta <\frac{1}{\delta_1} + \frac{1}{\beta_1\delta_1(1+c_{\Ta})\sqrt{r}} - \frac{\sqrt{1-\phi^2}}{\beta_1\delta_1\sqrt{r}}$, where $\phi = \beta_2c_{\H}\frac{2-\delta_2}{\sqrt{r}}-\beta_1(\delta_1 - 1)\sqrt{r} $ and $\beta_1>1, \beta_2 >0$ are constants. Then, AP-ROM outputs a sequence of estimates $L_t$ such that:    
\begin{align}\label{lineconveqAx}
\|L_{t+1} - L_*\|_2\leq{\leq}q'_1\|L_t-L_*\|_2  + q'_2Tr(L_*) +  \frac{1}{m}\left(\eta(1+C_{\Ta}) + \frac{\phi(1+C_{\H})}{\sqrt{1-\phi^2}}\right)\Big{\|}\P_{V_t}\A^*e\Big{\|}_F,
\end{align}
where $q'_1 =  \beta_1(1+C_{\Ta})(\eta\delta_1 - 1)\sqrt{r}+(1+c_{\Ta})\sqrt{1-\phi^2} < 1$ and $q'_2 = \beta_1(1+C_{\Ta}) \eta\delta_2\sqrt{r}$.
\end{theorem}
Similar to the corollary~\ref{Epinduc} and by the idea of fresh sampling, we have the following result.
\begin{corollary}\label{Apinduc}
Under the assumptions in Theorem~\ref{LinConAp}, in order to achieve $\epsilon$-accuracy in the estimation of $L_*$ in terms of spectral norm, AP-ROM requires $K = \O(\log(\frac{\|L_*\|_2}{\epsilon}))$ iterations. Specifically, the output of AP-ROM satisfies the following after $K$ iterations with high probability:
\begin{align}
\|L_{t+1} - L_*\|_2\leq{\leq}(q'_1)^K\|L_*\|_2  + \frac{q'_2}{1-q'_1}Tr(L_*) +\frac{q'_3}{1-q'_1}\sqrt{\frac{p\log p}{m}\log(\frac{p}{\xi_3})}.
\end{align}
where $q'_3 = \left(\eta(1+C_{\Ta}) + \frac{\phi(1+C_{\H})}{\sqrt{1-\phi^2}}\right).$
\end{corollary}
From Theorem~\ref{lineconveqAx} and Corollary~\ref{Apinduc}, we observe that the sample-complexity of AP-ROM (i.e., the number of samples $m$ to achieve a given accuracy) remains the same as in EP-ROM.




The above analysis of AP-ROM shows that instead of using exact rank-$r$ projections (as in EP-ROM), one can use instead tail and head approximate projection which is implemented by the randomized block Krylov method proposed by~\cite{musco2015randomized}. The running time for this method is given by $\widetilde{\O}(p^2r)$ if $r\ll p$ according to Theorem 7 in~\cite{musco2015randomized}. While the running time of the projection step is gap-independent, the calculation of the \emph{gradient} (i.e., the input to the head projection method $\H$) is itself a challenge. In essence, the bottleneck arises while calculating of the gradient is related to the calculation of the adjoint operator, $\A^*(d) = \sum_{i=1}^m d^{(i)}x_ix_i^T$ which requires $\O(p^2)$ operations for each sample. Coupled with the sample-complexity $m = \Omega(pr^2)$, this means that the running time per-iteration scaled as $\Omega(p^3 r^2)$, which overshadows any gains achieved during the projection step (please see the appendix for more discussion). 

To address this challenge, we propose a modified version of BK-SVD for head approximate projection which uses the special rank-one structures involved in the calculation of the gradients. We call this method \emph{Modified BK-SVD} or MBK-SVD. The basic idea is to efficiently evaluate the Krylov-subspace iteration of the BK-SVD in order to fully avoid any explicit calculations of the adjoint operator $\A^*$. Due to space constraints, the pseudocode as well as the analysis of MBK-SVD is deferred to the appendix.

\begin{theorem}\label{runningtimeAP}
To achieve $\epsilon$ accuracy, AP-ROM (with MBK-SVD) runs in time
$K = \widetilde{\O}\left(\max\left(p^2, p\|L_*\|_2^2\right)r^3\log^2(\frac{1}{\epsilon})\right)$ (Here, $\widetilde{\O}$ hides dependency on $polylog(p)$).
\end{theorem}

It is worthwhile to note that both the sample complexity and running time of AP-ROM depends on the spectral norm $\|L\|_*$. However, in several applications this can be assumed to be bounded above. For example, in learning two-layer polynomial networks, as discussed above, the relation between the output and input is given by $y = \sum_{j=1}^r \alpha_j \langle w_j, x \rangle^2 = x^T \left( \sum_j \alpha_j w_j w_j^T \right) x$ such that $|\alpha_j|\leq 1$ and $\|w_j\|_2=1$ (see~\cite{livni2014computational}, Section 4.3). Therefore, the spectral norm of $L_*$ is $O(r)$ in the worst case and \emph{constant} under reasonable incoherence assumptions on $w_i$. 


\begin{figure*}[t]
\begin{center}
\begin{tabular}{cccc}
\hskip -.12 in\includegraphics[trim = 8mm 65mm 15mm 65mm, clip, width=0.24\linewidth]{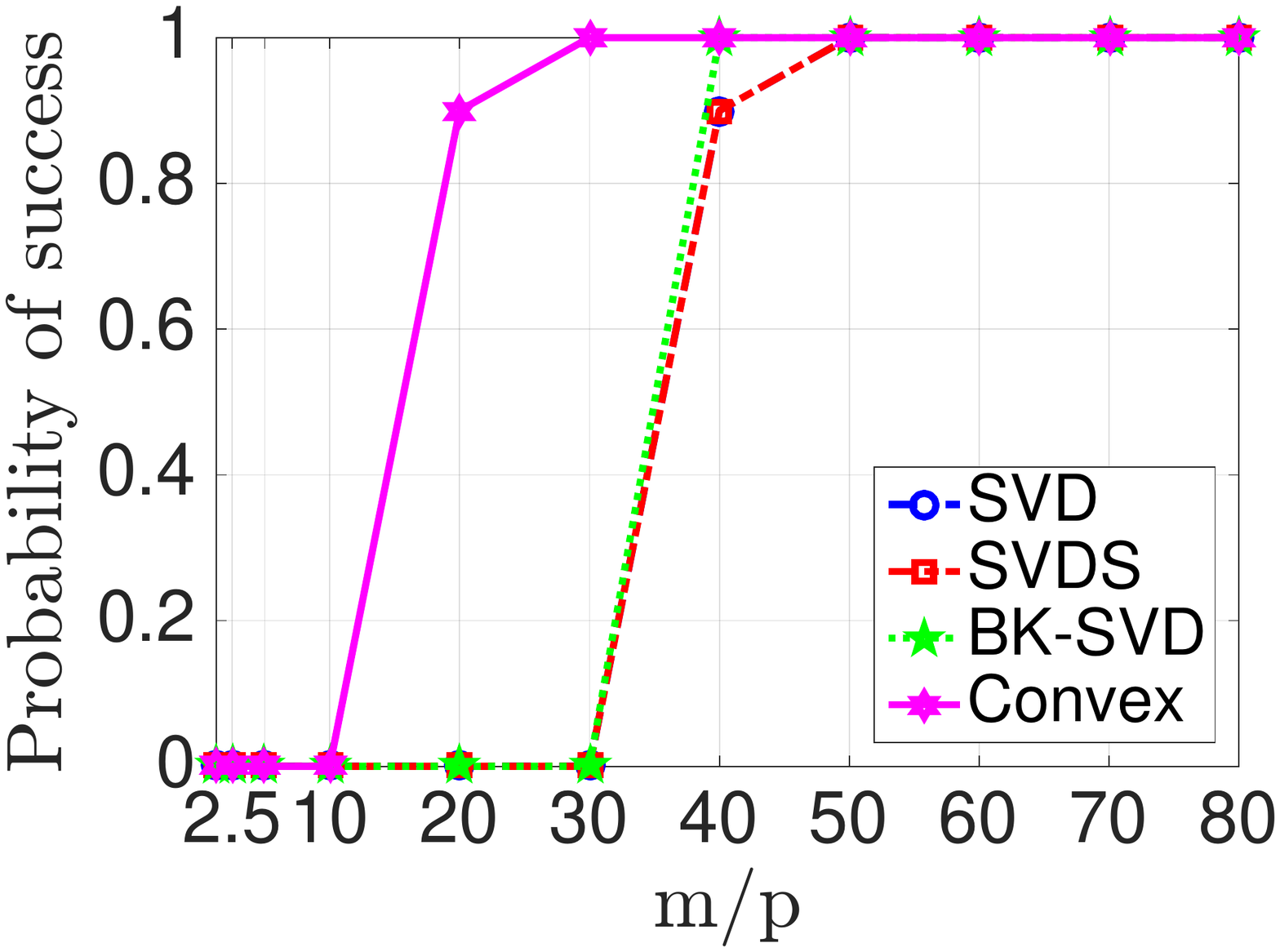} &
\includegraphics[trim = 8mm 65mm 15mm 65mm, clip, width=0.24\linewidth]{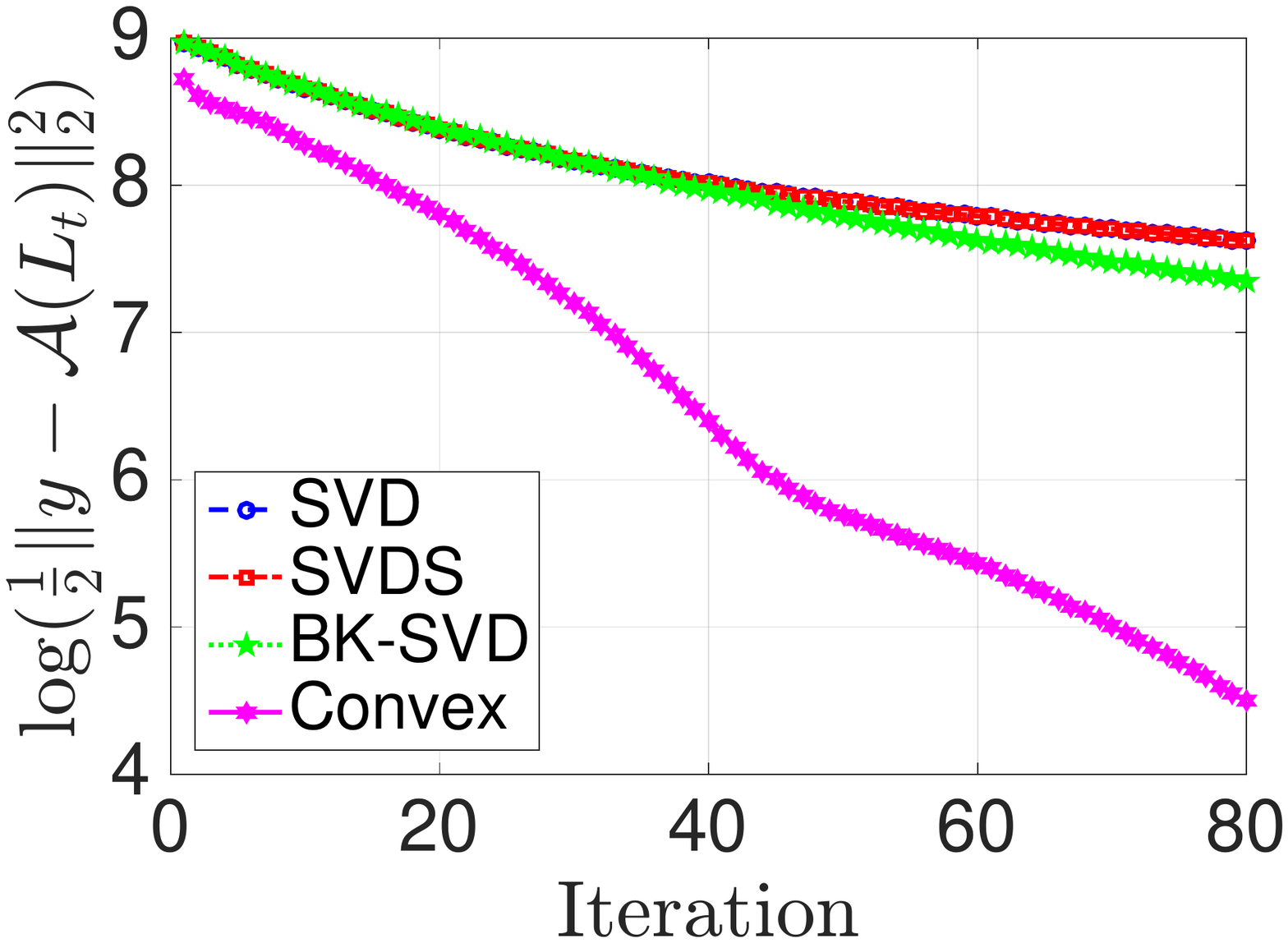} &
\includegraphics[trim = 8mm 65mm 15mm 65mm, clip, width=0.24\linewidth]{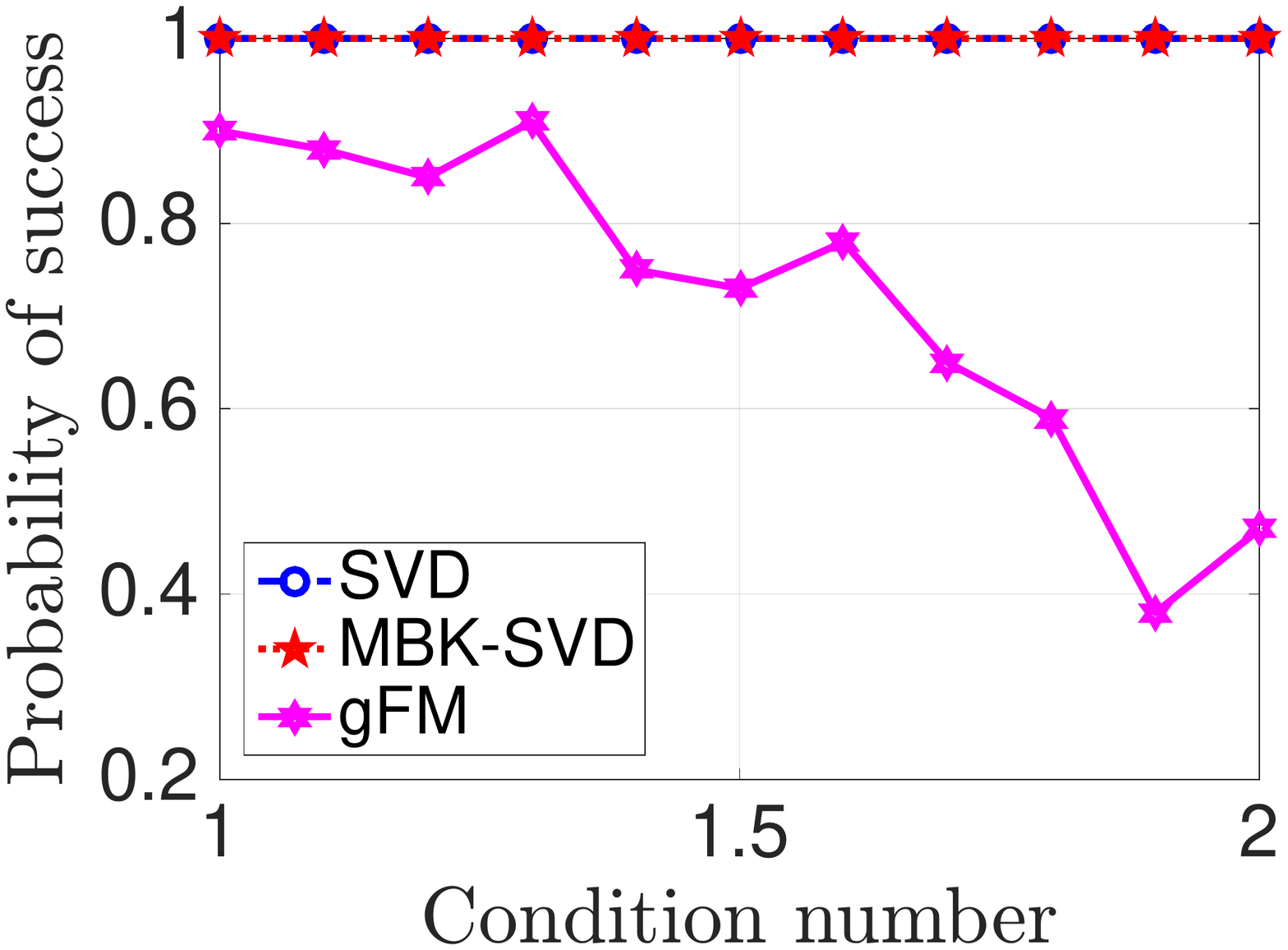} &
\includegraphics[trim = 8mm 65mm 11mm 65mm, clip, width=0.24\linewidth]{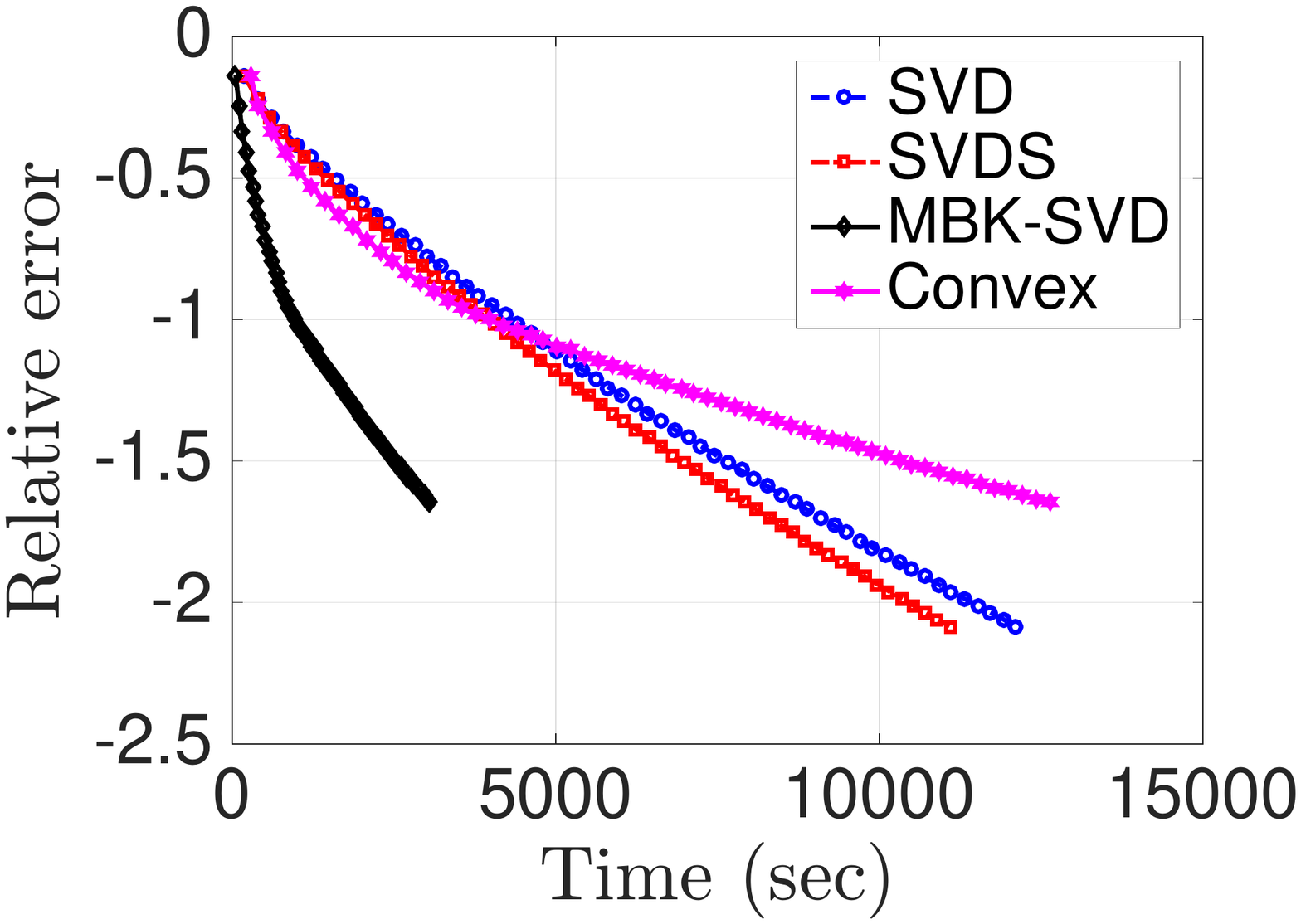}\\
$(a)$ & $(b)$ & $(c)$ & $(d)$
\end{tabular}
\vskip -0.10in\caption{\small  Comparison of algorithms. (a) Phase transition plot with $p =100$ and various values of $m$. (b) Evolution of the objective function versus number of iterations with $p =100$ and $m=5000$. (c) Probability of success in terms of condition number with $p=100$ and $m=600$. (d) Running time of the algorithm with $p=1000$ and $m=55000$.}
\vskip -0.2 in
\label{allfigs}
\end{center}
\end{figure*}

\section{Experimental results}
\label{sec:expe}
In this section, we illustrate some experiments to support our proposed algorithms. We compare EP-ROM and AP-ROM with convex nuclear norm minimization as well as the gFM algorithm of~\cite{lin2016non}. To solve the nuclear norm minimization, we use FASTA~\cite{GoldsteinStuderBaraniuk:2014,FASTA:2014} which efficiently implements an accelerated proximal sub-gradient method. For AP-ROM, we consider three different implementations. The first implementation of AP-ROM uses a low-rank approximation heuristic based on Lanczos iterations (implementable by Matlab's SVDS) instead of exact SVDs. The second implementation of AP-ROM uses the randomized Block Krylov SVD (BK-SVD) of~\cite{musco2015randomized}. Finally, the third AP-ROM implementation uses our proposed modified BK-SVD method. In all the experiments, we generate a low-rank matrix, $L_* = UU^T$, such that $\U\in\R^{p\times r}$ with $r=5$ where the entries of $U$ is randomly chosen according to the standard normal distribution. 

Figures~\ref{allfigs}(a) and~\ref{allfigs}(b) show the phase transition of successful recovery as well as the evolution of the objective function, $\frac{1}{2}\|y - \A(L_t)\|_2^2$ versus the iteration count $t$ for four algorithms (here, MBK-SVD has not been plotted since it behaves similar to BK-SVD in terms of accuracy). In these plots, we have used $10$ Monte Carlo trials and the phase transition plot is generated based on the empirical probability of success; here, success is when the relative error between $\hat{L}$ (the estimate of $L_*$) and the ground truth $L_*$ (measured in terms of spectral norm) is less than $0.05$. For solving convex nuclear norm minimization using FASTA, we set the Lagrangean parameter, $\mu$ i.e., $\mu\|L\|_* + \frac{1}{2}\|y - \A L\|_F^2$ via a grid search. In Figure~\ref{allfigs}(a), there is no additive noise; in Figure~\ref{allfigs}(b) we consider an additive standard normal noise with standard deviation equal to $0.1$. As we can see in Figure~\ref{allfigs}(a), the phase transition for convex method is slightly better than non-convex algorithms, which is consistent with known theoretical results. However, the convex method is \emph{improper}, i.e., the rank of $\hat{L}$ is much higher than $5$. 

In Figure \ref{allfigs}(c), we compare our algorithms with the gFM method of~\cite{lin2016non} with respect to the condition number of the ground truth matrix. The setup is as before, and we fix the number measurements as $m=6000$. We plot the probability of success when the relative error between $\hat{L}$ and ground truth $L_*$ is less than $0.05$ for $100$ Monte Carlo trials. This plot shows that the success of gFM heavily depends on condition number. (As per the theory, to be able to recover $L_*$, an additional factor proportional to the square of condition number is required.) Finally, in Figure \ref{allfigs}(d), we compare the algorithms in the high-dimensional regime where $p=1000$ and $r=5$ in terms of running time. Here, we have used our proposed modified BK-SVD algorithm. The y-axis denotes the relative error in spectral norm and we report averages over $10$ Monte Carlo trials. As we can see, convex methods are the slowest (as expected); the non-convex methods are comparable to each other, while MBK-SVD is the fastest. This plot verifies that our modified head approximate projection routine is almost $3$ times faster than other methods, which makes it a promising approach for high-dimensional matrix recovery applications.

\small
\bibliographystyle{unsrt}
\bibliography{../Common/mrsbiblio.bib,../Common/chinbiblio.bib,../Common/csbib.bib,../Common/kernels.bib}

\begin{thebibliography}{10}

\bibitem{fazel2002matrix}
M.~Fazel.
\newblock {\em Matrix rank minimization with applications}.
\newblock PhD thesis, PhD thesis, Stanford University, 2002.

\bibitem{recht2010guaranteed}
B.~Recht, M.~Fazel, and P.~Parrilo.
\newblock Guaranteed minimum-rank solutions of linear matrix equations via
  nuclear norm minimization.
\newblock {\em SIAM review}, 52(3):471--501, 2010.

\bibitem{zhong2015efficient}
K.~Zhong, P.~Jain, and I.~Dhillon.
\newblock Efficient matrix sensing using rank-1 gaussian measurements.
\newblock In {\em Int. Conf on Algorithmic Learning Theory}, pages 3--18.
  Springer, 2015.

\bibitem{kueng2017low}
R.~Kueng, H.~Rauhut, and U.~Terstiege.
\newblock Low rank matrix recovery from rank one measurements.
\newblock {\em Appl. Comput. Harmon. Anal.}, 42(1):88--116, 2017.

\bibitem{cai2015rop}
T.~Cai and A.~Zhang.
\newblock Rop: Matrix recovery via rank-one projections.
\newblock {\em Ann. Stat.}, 43(1):102--138, 2015.

\bibitem{chen2015exact}
Y.~Chen, Y.~Chi, and A.~Goldsmith.
\newblock Exact and stable covariance estimation from quadratic sampling via
  convex programming.
\newblock {\em {IEEE} Trans. Inform. Theory}, 61(7):4034--4059, 2015.

\bibitem{livni2014computational}
R.~Livni, S.~Shalev-Shwartz, and O.~Shamir.
\newblock On the computational efficiency of training neural networks.
\newblock In {\em Adv. Neural Inf. Proc. Sys. (NIPS)}, pages 855--863, 2014.

\bibitem{lin2016non}
M.~Lin and J.~Ye.
\newblock A non-convex one-pass framework for generalized factorization machine
  and rank-one matrix sensing.
\newblock In {\em Adv. Neural Inf. Proc. Sys. (NIPS)}, pages 1633--1641, 2016.

\bibitem{shalev2011large}
S.~Shalev-shwartz, A.~Gonen, and O.~Shamir.
\newblock Large-scale convex minimization with a low-rank constraint.
\newblock In {\em Proc. Int. Conf. Machine Learning}, pages 329--336, 2011.

\bibitem{musco2015randomized}
C.~Musco and C.~Musco.
\newblock Randomized block krylov methods for stronger and faster approximate
  singular value decomposition.
\newblock In {\em Adv. Neural Inf. Proc. Sys. (NIPS)}, pages 1396--1404, 2015.

\bibitem{HegdeFastUnionNips2016}
C.~Hegde, I.~Indyk, and S.~Ludwig.
\newblock Fast recovery from a union of subspaces.
\newblock In {\em Adv. Neural Inf. Proc. Sys. (NIPS)}, 2016.

\bibitem{jain2010guaranteed}
P.~Jain, R.~Meka, and I.~Dhillon.
\newblock Guaranteed rank minimization via singular value projection.
\newblock In {\em Adv. Neural Inf. Proc. Sys. (NIPS)}, pages 937--945, 2010.

\bibitem{cai2010singular}
J.~Cai, E.~Cand{\`e}s, and Z.~Shen.
\newblock A singular value thresholding algorithm for matrix completion.
\newblock {\em SIAM Journal on Optimization}, 20(4):1956--1982, 2010.

\bibitem{davenport2016overview}
M.~Davenport and J.~Romberg.
\newblock An overview of low-rank matrix recovery from incomplete observations.
\newblock {\em {IEEE} J. Select. Top. Sig. Proc.}, 10(4):608--622, 2016.

\bibitem{candes2008introduction}
E.~Cand{\`e}s and M.~Wakin.
\newblock An introduction to compressive sampling.
\newblock {\em {IEEE} Sig. Proc. Magazine}, 25(2):21--30, 2008.

\bibitem{foucart2013}
S.~Foucart and H.~Rauhut.
\newblock {\em A mathematical introduction to compressive sensing}, volume~1.
\newblock Springer.

\bibitem{jung2013compressive}
A.~Jung, G.~Taub{\"o}ck, and F.~Hlawatsch.
\newblock Compressive spectral estimation for nonstationary random processes.
\newblock {\em {IEEE} Trans. Inform. Theory}, 59(5):3117--3138, 2013.

\bibitem{liu2011universal}
Y.~Liu.
\newblock Universal low-rank matrix recovery from pauli measurements.
\newblock In {\em Adv. Neural Inf. Proc. Sys. (NIPS)}, pages 1638--1646, 2011.

\bibitem{candes2011tight}
E.~Candes and Y.~Plan.
\newblock Tight oracle inequalities for low-rank matrix recovery from a minimal
  number of noisy random measurements.
\newblock {\em {IEEE} Trans. Inform. Theory}, 57(4):2342--2359, 2011.

\bibitem{Secondorder17}
M.~Lin, S.~Qiu, B.~Hong, and J.~Ye.
\newblock The second order linear model.
\newblock {\em arXiv preprint arXiv:1703.00598}, 2017.

\bibitem{candes2009exact}
E.~Cand{\`e}s and B.~Recht.
\newblock Exact matrix completion via convex optimization.
\newblock {\em Found. Comput. Math.}, 9(6), 2009.

\bibitem{dasarathy2015sketching}
G.~Dasarathy, P.~Shah, B.~Bhaskar, and R.~Nowak.
\newblock Sketching sparse matrices, covariances, and graphs via tensor
  products.
\newblock {\em {IEEE} Trans. Inform. Theory}, 61(3):1373--1388, 2015.

\bibitem{candes2013phaselift}
E.~Candes, T.~Strohmer, and V.~Voroninski.
\newblock Phaselift: Exact and stable signal recovery from magnitude
  measurements via convex programming.
\newblock {\em Comm. Pure Appl. Math.}, 66(8):1241--1274, 2013.

\bibitem{netrapalli2013phase}
P.~Netrapalli, P.~Jain, and S.~Sanghavi.
\newblock Phase retrieval using alternating minimization.
\newblock In {\em Adv. Neural Inf. Proc. Sys. (NIPS)}, pages 2796--2804, 2013.

\bibitem{boyd2004convex}
S.~Boyd and L.~Vandenberghe.
\newblock {\em Convex optimization}.
\newblock Cambridge university press, 2004.

\bibitem{hegde2015approximation}
C.~Hegde, P.~Indyk, and L.~Schmidt.
\newblock Approximation algorithms for model-based compressive sensing.
\newblock {\em IEEE Trans. Inform. Theory}, 61(9):5129--5147, 2015.

\bibitem{lanczos1950iteration}
C.~Lanczos.
\newblock An iteration method for the solution of the eigenvalue problem of
  linear differential and integral operators1.
\newblock {\em Journal of Research of the National Bureau of Standards}, 45(4),
  1950.

\bibitem{clarksonwoodruff_old}
K.~L Clarkson and D.~Woodruff.
\newblock Low rank approximation and regression in input sparsity time.
\newblock In {\em Proc. ACM Symp. Theory of Comput.}, pages 81--90. ACM, 2013.

\bibitem{drineas_mahoney}
M.~Mahoney and P.~Drineas.
\newblock Cur matrix decompositions for improved data analysis.
\newblock {\em Proc. Natl. Acad. Sci.}, 106(3):697--702, 2009.

\bibitem{tygert}
V.~Rokhlin, A.~Szlam, and M.~Tygert.
\newblock A randomized algorithm for principal component analysis.
\newblock {\em {SIAM} J. Matrix Anal. Appl.}, 31(3):1100--1124, 2009.

\bibitem{GoldsteinStuderBaraniuk:2014}
T.~Goldstein, C.~Studer, and R.~Baraniuk.
\newblock A field guide to forward-backward splitting with a {FASTA}
  implementation.
\newblock {\em arXiv eprint}, abs/1411.3406, 2014.

\bibitem{FASTA:2014}
T.~Goldstein, C.~Studer, and R.~Baraniuk.
\newblock {FASTA}: A generalized implementation of forward-backward splitting,
  January 2015.
\newblock http://arxiv.org/abs/1501.04979.

\bibitem{vershynin2010introduction}
R.~Vershynin.
\newblock Introduction to the non-asymptotic analysis of random matrices.
\newblock {\em arXiv preprint arXiv:1011.3027}, 2010.

\bibitem{tropp2015introduction}
J.~Tropp.
\newblock An introduction to matrix concentration inequalities.
\newblock {\em Foundations and Trends{\textregistered} in Machine Learning},
  8(1-2):1--230, 2015.

\end{thebibliography}

\section{Appendix}
\label{sec:appen}

Below, the expression $C+D$ for two sets $C$ and $D$ refers to the \textit{Minkowski} sum of two sets, defined as $C+D = \{c+d \ | \ c\in C, \ d\in D\}$ for given sets $C$ and $D$. Also, $\M(\U_r)$ denotes the set of vectors associated with $\U_r$, the set of all rank-r matrix subspaces. In addition, we denote the $\{p,q\}^{th}$ entry of matrix $B$ and $p^{th}$ entry of vector $x$ as $B^{(pq)}$ and $x^{(p)}$, respectively. 


\begin{proof}[Proof of Theorem~\ref{LinConEX}]
Here we show that the error in the estimate of $L_*$ decreases in one iteration. Let $V^t, V^{t+1}$, and $V^{*}$ denote the bases for the column space of $L^t, L^{t+1}$, and $L^*$, respectively. Define set $J$ as $V^t\cup V^{t+1}\cup V^*\subseteq J_t = J$ and $P_J$ as projection onto it. In addition, define $b$ as follows:
$$b = L_t -\eta\P_J\left(\frac{1}{m}\sum_{i=1}^{m}\left( x_iL_tx_i^T-y_i\right) x_ix_i^T - Tr(L_t-\bar{y})I\right) = L_t - \frac{\eta}{m}\P_J\A^*(\A(L_t) -y) + \eta \P_JTr(L_t-\bar{y})I.$$
We have:
\begin{align}\label{EpLinproof}
\|L_{t+1} - L_*\|_2&\leq\|L_{t+1} - b\|_2 + \|b-L_*\|_2\overset{a_1}{\leq}2\|b-L_*\|_2 \nonumber \\
&\leq 2\|L_t - L_*- \frac{\eta}{m}\P_J\A^*(\A(L_t) -y) + \eta \P_JTr(L_t-\bar{y})I \|_2\nonumber \\
&\overset{a_2}{\leq}2\|\P_J\left(L_t - L_*- \frac{\eta}{m}A^*\A(L_t-L_*) + \eta Tr(L_t-\bar{y})I\right) \|_2 + \frac{2\eta}{m}\|\P_J\A^*e\|_2\nonumber \\
&\overset{a_{3}}{\leq}2\|L_t - L_*- \frac{\eta}{m}A^*\A(L_t-L_*) + \eta Tr(L_t-\bar{y})I \|_2 + \frac{2\eta}{m}\|A^*e\|_2\nonumber \\
&\overset{a_{4}}{\leq}2(\eta\delta_1 - 1)\|L_t-L_*\|_2  +2 \eta\delta_2Tr(L_*)+ \frac{2\eta}{m}\|\A^*e\|_2,
\end{align}
where  $a_1$ holds since $L_{t+1}$ is generated by projecting onto the set of matrices with rank $r$, and by definition of $J$, $L_{t+1}$ also has the minimum Euclidean distance to b over all matrices with rank $r$; $a_2$ holds by the definition of $y$ from~\eqref{Obmodel} and the triangle inequality; $a_3$ is followed by the fact that $L_t - L_*$ lies in set $J$ and due to the spectral norm of projection matrix $\P_J$; finally, $a_4$ holds by the CU-RIP assumption in the theorem. By choosing $q_1 = 2\eta\delta_1 - 1$ and $q_2 =2 \eta\delta_2$, the proof is completed.
\end{proof}

\begin{proof}[Proof of Corollary~\ref{Epinduc}]
First, we note that by our assumption on $\eta$ in Theorem~\ref{LinConEX}, $q_1<1$. Since EP-ROM uses fresh samples in each iteration, $L_t - L_*$ is independent of the sensing vectors, $x_i$'s for all $t$. On the other hand, from Theorem~\ref{CURIP}, the CU-RIP holds with probability $1-\xi$. As a result, by a union bound over the $K$ iterations of the algorithm, the CU-RIP holds after $K$ iterations with probability at least $1-K\xi$. By recursively applying inequality~\eqref{lineconveqEx} (with zero initialization) and applying Theorem~\ref{Staterror}, we obtain the claimed result. 
\end{proof}

\begin{proof}[Proof of Theorem~\ref{LinConAp}]
Assume that $Y\in\M(\U_{2r})$ such that $L_t-L_*\in Y$ and $V := V_t = \H\left(\A^*(\A(L_t) -y) -Tr(L_t-\bar{y})I\right)$. Also, define 
$$b' = L_t -\eta\H\left(\frac{1}{m}\sum_{i=1}^{m}\left( x_iL_tx_i^T-y_i\right) x_ix_i^T - Tr(L_t-\bar{y})I\right) = L_t - \frac{\eta}{m}\H\left(\A^*(\A(L_t) -y) -Tr(L_t-\bar{y})I\right).$$
Furthermore, by definition of approximate tail projection, $L_t\in\M(\U_r)$. Now, we have:
\begin{align}
\|L_{t+1} - L_*\|_F &=\| L_* - \Ta(b^{\prime})\|_F \nonumber \\
&\leq\|L^* - b^{\prime}\|_F + \|b^{\prime} - \Ta(b^{\prime})\|_F \nonumber \\
&\overset{a_1}{\leq}(1+c_{\Ta})\|b^{\prime} - L_*\|_F  \nonumber \\
&\overset{}{=}(1+c_{\Ta})\Big{\|}L_t - L_* - \eta\H\left(\frac{1}{m}\A^*(\A(L_t) -y) -Tr(L_t-\bar{y})I\right)\Big{\|}_F \nonumber \\
&\overset{a_2}{=}(1+c_{\Ta})\Big{\|}L_t - L_* - \eta\P_V\left(\frac{1}{m}\A^*(\A(L_t) -y) -Tr(L_t-\bar{y})I\right)\Big{\|}_F, \nonumber 
\end{align}
where $a_1$ is implied by the triangle inequality and the definition of approximate tail projection, and inequality $a_2$ holds by the definition of approximate head projection. Next, we have:
\begin{align}
\label{headtailproofROMS}
\|L_{t+1} &- L_*\|_F\\
&\overset{a_3}{\leq}(1+c_{\Ta})\Big{\|}\P_V(L_t-L_*) + \P_{V^{\bot}}(L_t-L_*) - \eta\P_V\left(\frac{1}{m}\A^*(\A(L_t) -y) -Tr(L_t-\bar{y})I\right)\Big{\|}_F \nonumber \\
&\overset{a_4}{\leq}(1+c_{\Ta})\Big{\|}\P_V(L_t-L_*) - \eta\P_V\left(\frac{1}{m}\A^*\A(L_t -L_*) -Tr(L_t-\bar{y})I\right)\Big{\|}_F\nonumber \\
&\hspace{40mm}+ (1+c_{\Ta})\Big{\|}\P_{V^{\bot}}(L_t-L_*)\Big{\|}_F +\frac{\eta(1+c_{\Ta})}{m}\Big{\|}\P_V\A^*e\Big{\|}_F \nonumber\\
&\overset{a_5}{\leq}(1+c_{\Ta})\Big{\|}\P_{V+Y}\left(L_t-L_* - \eta\left(\frac{1}{m}\A^*\A(L_t -L_*) -Tr(L_t-\bar{y})I\right)\right)\Big{\|}_F \nonumber\\
&\hspace{40mm} + (1+c_{\Ta})\Big{\|}\P_{V^{\bot}}(L_t-L_*)\Big{\|}_F+  \frac{\eta(1+c_{\Ta})}{m}\Big{\|}\P_V\A^*e\Big{\|}_F,
\end{align}
where $a_3$ follows by decomposing the residual $L_t - L_*$ on the two subspaces $V$ and $V^{\bot}$, and $a_4$ is due to the triangle inequality and the fact that $L_t - L_*\in Y$ and $V\subseteq V+Y$. 

Now, we need to bound the three terms in~\eqref{headtailproofROMS}. The third (statistical error) term can be bounded by using Theorem~\ref{Staterror} which we will use in Corollary~\ref{Apinduc}. For the first term, we have:
\begin{align}\label{FiAP}
&(1+c_{\Ta})\Big{\|}\P_{V+Y}\left(L_t-L_* - \eta\left(\frac{1}{m}\A^*\A(L_t -L_*) -Tr(L_t-\bar{y})I\right)\right)\Big{\|}_F\nonumber \\
&\hspace{40mm}\overset{a_1}{\leq}\beta_1(1+C_{\Ta})\sqrt{r}\Big{\|}L_t-L_* - \eta\left(\frac{1}{m}\A^*\A(L_t -L_*) -Tr(L_t-\bar{y})I\right)\Big{\|}_2\nonumber \\
&\hspace{40mm}\overset{a_{2}}{\leq}\beta_1(1+C_{\Ta})(\eta\delta_1 - 1)\sqrt{r}\n L_t-L_*\n_2  +\beta_1(1+C_{\Ta}) \eta\delta_2\sqrt{r}Tr(L_*) \nonumber \\
&\hspace{40mm}\overset{}{\leq}\beta_1(1+C_{\Ta})(\eta\delta_1 - 1)\sqrt{r}\n L_t-L_*\n_F  +\beta_1(1+C_{\Ta}) \eta\delta_2\sqrt{r}Tr(L_*),
\end{align}
where $a_1$ holds by the properties of Frobenius and spectral norm and $a_2$ is due to the CU-RIP assumption in the theorem similar to~\eqref{EpLinproof}. To bound second term in~\eqref{headtailproofROMS}, $(1+c_{\Ta})\Big{\|}\P_{V^{\bot}}(L_t-L_*)\Big{\|}_F$, we give upper and lower bounds for $\Big{\|}\P_V\left(\frac{1}{m}\A^*(\A(L_t) -y) -Tr(L_t-\bar{y})I\right)\Big{\|}_F$ as follows:  
\begin{align}\label{SeAPLow}
\Big{\|}\P_V &\left(\frac{1}{m}\A^*(\A(L_t) -y) -Tr(L_t-\bar{y})I\right)\Big{\|}_F \\
&\overset{a_1}{\geq} C_{\H}\Big{\|}\P_Y\left(\frac{1}{m}\A^*(\A(L_t) -y) -Tr(L_t-\bar{y})I\right)\Big{\|}_F \nonumber \\
&\overset{a_2}{\geq}C_{\H}\Big{\|}\P_Y\left(\frac{1}{m}\A^*\A(L_t -L_*) -Tr(L_t-\bar{y})I\right)\Big{\|}_F - \frac{C_{H}}{m}\Big{\|}\P_V\A^*e\Big{\|}_F\nonumber \\
&\overset{}{\geq}C_{\H}\Big{\|}\P_Y\left(\frac{1}{m}\A^*\A(L_t -L_*) -Tr(L_t-\bar{y})I\right)\Big{\|}_2 - \frac{C_{H}}{m}\Big{\|}\P_V\A^*e\Big{\|}_F\nonumber \\
&\overset{a_3}{\geq}C_{\H}(2-\delta_2)\Big{\|}L_t-L_*\Big{\|}_2 - \delta_2Tr(L_*)- \frac{C_{H}}{m}\Big{\|}\P_V\A^*e\Big{\|}_F\nonumber \\
&\overset{a_4}{\geq}\beta_2C_{\H}\frac{2-\delta_2}{\sqrt{r}}\Big{\|}L_t-L_*\Big{\|}_F - \delta_2Tr(L_*)- \frac{C_{H}}{m}\Big{\|}\P_V\A^*e\Big{\|}_F,
\end{align}
where $a_1$ holds by the definition of tail approximate projection, $a_2$ is followed by triangle inequality, $a_3$ is due to Corollary~\ref{CURIPConseq}, and finally $a_4$ holds due to the fact that $\rank(L_t-L_*)\leq2r$. Moreover, for the upper bound, we have:
\begin{align}\label{SeAPUp}
&\Big{\|}\P_V\left(\frac{1}{m}\A^*(\A(L_t) -y) -Tr(L_t-\bar{y})I\right)\Big{\|}_F\nonumber \\
&\hspace{3mm}\overset{a_1}{\leq}\Big{\|} \P_{V+Y}\left(\frac{1}{m}\A^*\A(L_t-L_*) -Tr(L_t-\bar{y})I\right) -\P_{V+Y}(L_t-L_*) \Big{\|}_F + \Big{\|}\P_V(L_t-L_*)\Big{\|}_F+\frac{1}{m}\Big{\|}\P_V\A^*e\Big{\|}_F\nonumber \\
&\hspace{3mm}\overset{a_2}{\leq}\Big{\|}L_t-L_*- \frac{1}{m}\A^*(\A(L_t) -y) +Tr(L_t-\bar{y})I\Big{\|}_F +\Big{\|}\P_V(L_t-L_*)\Big{\|}_F+\frac{1}{m}\Big{\|}\P_V\A^*e\Big{\|}_F\nonumber \\
&\hspace{3mm}\overset{a_3}{\leq}\beta_1(\delta_1 - 1)\sqrt{r}\Big{\|}L_t-L_*\Big{\|}_F  +\beta_1\delta_2\sqrt{r}Tr(L_*)+\Big{\|}\P_V(L_t-L_*)\Big{\|}_F+\frac{1}{m}\Big{\|}\P_V\A^*e\Big{\|}_F,
\end{align}
where $a_1$ holds by triangle inequality and the fact that projection onto the extended subspace $V+Y$ ($V\subseteq V+Y$) does not decrease the Frobenius norm, $a_2$ is due to the inequality $\|AB\|_F\leq\|A\|_2\|B\|_F$ and spectral norm of projection operator, and finally $a_3$ is followed by CU-RIP assumption. Putting together~\eqref{SeAPLow} and \eqref{SeAPUp}, we obtain:
\begin{align}
\Big{\|}\P_V(L_t-L_*)\Big{\|}_F\geq\left(\beta_2C_{\H}\frac{2-\delta_2}{\sqrt{r}}-\beta_1(\delta_1 - 1)\sqrt{r} \right)\Big{\|}L_t-L_*\Big{\|}_F -\frac{1+C_{\H}}{m}\Big{\|}\P_V\A^*e\Big{\|}_F.
\end{align}

By Pythagoras theorem, we know $\Big{\|}\P_V(L_t-L_*)\Big{\|}_F^2 + \Big{\|}\P_{V^\bot}(L_t-L_*)\Big{\|}_F^2 = \Big{\|}L_t-L_*\Big{\|}_F^2$, hence we can bound the second term in~\eqref{headtailproofROMS}. To use this fact, we apply claim (14) in \cite{HegdeFastUnionNips2016} which results:
\begin{align}\label{SeAP}
(1+c_{\Ta})\Big{\|}\P_V(L_t-L_*)\Big{\|}_F\leq (1+c_{\Ta})\sqrt{1-\phi^2}\Big{\|}L_t-L_*\Big{\|}_F + \frac{\phi(1+C_{\H})}{m\sqrt{1-\phi^2}}\Big{\|}\P_V\A^*e\Big{\|}_F.
\end{align}
where $\phi =\beta_2 C_{\H}\frac{2-\delta_2}{\sqrt{r}}-\beta_1(\delta_1 - 1)\sqrt{r} $.
Putting all the bounds in~\eqref{FiAP}, and \eqref{SeAP} altogether, we obtain:
\begin{align}\label{ApLinproof}
\|L_{t+1} - L_*\|_F &\leq\left( \beta_1(1+C_{\Ta})(\eta\delta_1 - 1)\sqrt{r}+(1+c_{\Ta})\sqrt{1-\phi^2}\right)\Big{\|}L_t-L_*\Big{\|}_F \nonumber \\
&\hspace{15mm} + \beta_1(1+C_{\Ta}) \eta\delta_2\sqrt{r}Tr(L_*) + \frac{1}{m}\left(\eta(1+C_{\Ta}) + \frac{\phi(1+C_{\H})}{\sqrt{1-\phi^2}}\right)\Big{\|}\P_V\A^*e\Big{\|}_F\nonumber \\
& = q'_1\n L_{t} - L_*\n _F + q'_2Tr(L_*) + \frac{1}{m}\left(\eta(1+C_{\Ta}) + \frac{\phi(1+C_{\H})}{\sqrt{1-\phi^2}}\right)\Big{\|}\P_V\A^*e\Big{\|}_F.
\end{align}
by choosing $q'_1= \beta_1(1+C_{\Ta})(\eta\delta_1 - 1)\sqrt{r}+(1+c_{\Ta})\sqrt{1-\phi^2}$, and $q'_2 =  \beta_1(1+C_{\Ta}) \eta\delta_2\sqrt{r}$, the proof is completed.
\end{proof}

In the proof of Theorem~\ref{LinConAp}, we have implicitly used the assumption on $\beta_1>1$ in the choice of $\rho_1 =\eta$. Since $\beta_1>1$, then we have $\beta_1 >\frac{1}{(1+C_{\Ta})\sqrt{r}}-\frac{\sqrt{1-\phi^2}}{\sqrt{r}}$. This condition forces $\eta<\frac{2}{\delta_1}$ which is necessary for using CU-RIP. Also, the condition on $1<\delta_1<2$ justifies using the CU-RIP condition in deriving inequality \eqref{SeAPUp}.

\begin{proof}[proof of corollary~\ref{Apinduc}]
The proof is similar to corollary~\ref{Epinduc} and it follows by using CU-RIP over $K$ iterations which is guaranteed to be held by using fresh samples in each iteration. Finally, by using induction, zero initialization, and Theorem~\ref{Staterror}, we obtain the claimed result in the corollary.
\end{proof}

\subsection{Supporting lemmas and theorems}
For proving lemmas in section~\ref{verifiCURIP}, we need some definitions and well-known Bernstein inequalities for random variables and matrices. We restate these inequalities for completeness. Please see \cite{vershynin2010introduction,tropp2015introduction} for more details.  
\begin{definition}(Subgaussian and Subexponential random variables.)\label{subexp}
A random variable $X$ is called subgaussian if it satisfies the following:
\begin{align*}
\mathbb{E}\exp\left(\frac{c X^2}{\|X\|_{\psi_2}^2}\right)\leq 2,
\end{align*}
where $\|X\|_{\psi_2}$ denotes the $\psi_2$-norm which is defined as follows:
\begin{align*}
\|X\|_{\psi_2}=\sup_{p\geq 1}\frac{1}{\sqrt{p}}(\mathbb{E}|X|^p)^{\frac{1}{p}}.
\end{align*}
Furthermore, a random variable $X$ is subexponential if it satisfies the following relation:
\begin{align*}
\mathbb{E}\exp\left(\frac{c X}{\|X\|_{\psi_1}}\right)\leq 2,
\end{align*}
where $\|X\|_{\psi_1}$ denotes the $\psi_1$-norm, defined as follows:
\begin{align*}
\|X\|_{\psi_1}=\sup_{p\geq 1}\frac{1}{p}(\mathbb{E}|X|^p)^{\frac{1}{p}}.
\end{align*}
In above expressions $c >0$ is an absolute constant.
\end{definition}
We note that the product of two standard normal random variables which is a $\chi^2$ random variable satisfies the subexponential random variable definition with $\psi_1$-norm equals to 2.
\begin{lemma}(Bernstein-type inequality for random variables).\label{bernVar}
Let $X_1,X_2,\ldots,X_n$ be independent sub-exponential random variables with zero-mean. Also, assume that $K = \max_{i}\|X_i\|_{\psi_1}$. Then, for any vector $a\in\mathbb{R}^n$ and every $t\geq 0$, we have:
\begin{align*}
\mathbb{P}(|\Sigma_{i}a_iX_i|\geq t)\leq 2\exp\left(-c\min\left\{\frac{t^2}{K^2\|a\|_2^2},\frac{t}{K\|a\|_\infty}\right\}\right).
\end{align*} 
where $c>0$ is an absolute constant.
\end{lemma}

\begin{lemma}(Bernstein-type inequality for symmetric random matrices). \label{bernMat}
Consider a sequence of symmetric and random independent identical distributed matrices $\{S_i\}_{i=1}^{m}$ with dimension $p\times p$. Also, assume that $\|S_i-\E S_i\|_2\leq R$ for $i=1,\ldots,m$. Then for all $t\geq 0$,
\begin{align*}
\pr\left(\Big{\|}\frac{1}{m}\sum_{i=1}^mS_i - \E S_i\Big{\|}_2\geq t\right)\leq 2p\exp\left(\frac{-mt^2}{\sigma +Rt/3}\right),
\end{align*}
where $\sigma = \|\E\left(S -\E S\right)^2\|_2$ and $S$ is a independent copy of $S_i$'s.
\end{lemma}

\subsection{Verification of CU-RIP$(\rho_1,\rho_2,\rho_3)$}
\label{verifiCURIP}
Before verifying of CU-RIP, we need the following lemma. In the first lemma, we show that $\bar{y} = \frac{1}{m}\sum_{i=1}^{m}y_i$ is concentrated around its mean with high probability.

\begin{lemma}[Concentration of $\bar{y}$]\label{TrMainCon}
Let $y$ be the measurement vector defined as~\eqref{Obmodel}. Then with probability $1-\xi_1$, we have for some constant $C>0$: 
\begin{align}
|\bar{y} - Tr(L_*)|\leq C\sqrt{\frac{1}{m}\log(\frac{p}{\xi_1})}Tr(L_*).
\end{align}
\end{lemma}
\begin{proof}
In all the following expressions, $c_l>0$ for $l=1,\ldots,4$ are absolute constants. First we note that: 
$$\E\bar{y} = \E y_1 = \E (Tr(x_ix_i^TL_*) + e_1) = Tr(L_*).$$
where we have used the i.i.d and zero-mean assumption of the noise vector $e$ and $x_i\overset{i.i.d}{\thicksim}\mathcal{N}(0,I)$. We have for all $t>0$:
\begin{align}\label{Trmain}
\pr\left(\big{|}\bar{y} - Tr(L_*)\big{|}\geq t\right) &= \pr\left(\Big{|}\frac{1}{m}\sum_{i=1}^{m}\langle x_ix_i^T,L_*\rangle - Tr(L_*)\Big{|}\geq t\right)\nonumber\\
&=\pr\left(\Big{|}\frac{1}{m}\sum_{i=1}^{m}\sum_{u,v}(x_i^{u}x_i^{v}L_*^{uv}) - Tr(L_*)\Big{|}\geq t\right)\nonumber\\
&=\pr\left(\Big{|}\sum_{u}\frac{1}{m}\sum_{i=1}^m((x_i^{u})^2L_*^{uu}-L_*^{uu}) + \sum_{u\neq v}\frac{1}{m}\sum_{i=1}^m(x_i^{u}x_i^{v}L_*^{uv})\Big{|}\geq t\right).
\end{align}
Now we bound two probabilities. First, $\forall~t_1\geq 0$:
\begin{align*}
\pr\left(\Big{|}\sum_{u}\frac{1}{m}\sum_{i=1}^m((x_i^{u})^2L_*^{uu}-L_*^{uu})\Big{|}\geq t_1\right)\overset{a_1}{\leq}p\exp\left({-c_1\frac{mt^2}{(Tr(L_*))^2}}\right),
\end{align*}
where $a_1$ is due to the union bound over $p$ diagonal variables and by the fact that $(x_i^{u})^2$ is a $\chi^2$ random variable with mean 1 and $\|\chi^2\|_{\psi_1} = 2$; as a result, we can use the scalar version of Bernstein inequality in~\eqref{bernVar}. Now by choosing $t_1\geq c_2Tr(L_*)\sqrt{\frac{\log(\frac{p}{\xi'_1})}{m}}$, with probability at least $1-\xi'_1$, we have:
\begin{align}\label{TrOne}
\Big{|}\sum_{u}\frac{1}{m}\sum_{i=1}^m((x_i^{u})^2L_*^{uu}-L_*^{uu})\Big{|}\leq Tr(L_*)\sqrt{\frac{c_2}{m}\log(\frac{p}{\xi'_1})}.
\end{align}
Second, let $k = \max_{u\neq v}(L_*^{uv})^2$. Thus, $\forall t_2\geq 0$,
\begin{align*}
\pr\left(\Big{|}\sum_{u\neq v}\frac{1}{m}\sum_{i=1}^m(x_i^{u}x_i^{v}L_*^{uv})\Big{|}\geq t_2\right)\overset{a_2}{\leq}p^2\exp\left({-c_2\frac{mt_2^2}{k^2}}\right),
\end{align*}
where $a_2$ holds by a union bound over $p^2-p$ off-diagonal variables. By the fact that $x_i^{u}x_i^{v}$ is a zero mean subexponential random variable. Hence, we can again use the scalar version of Bernstein inequality in~\eqref{bernVar}. By choosing $t_2\geq \sqrt{\frac{c_3}{m}\log(\frac{p}{\xi_1^{''}})}$, with probability at least $1-\xi^{''}_1$, we have:
\begin{align}\label{TrTwo}
\Big{|}\sum_{u\neq v}\frac{1}{m}\sum_{i=1}^m(x_i^{u}x_i^{v}L_*^{uv})\Big{|}\leq\sqrt{\frac{c_3}{m}\log(\frac{p}{\xi_1^{''}})}.
\end{align}
Now from~\eqref{Trmain}, \eqref{TrOne}, and~\eqref{TrTwo} and by choosing $t = t_1 + t_2$ with probability at least $1-\xi_1$ where $\xi_1 = \xi^{'}_1+ \xi^{''}_1$, we obtain:
$$\pr\left(\big{|}\bar{y} - Tr(L_*)\big{|}\geq t\right)\leq \sqrt{\frac{c_4}{m}\log(\frac{p}{\xi'_1})}Tr(L_*). $$
which proves the claim in the lemma.
\end{proof}

In the next lemma, we show that $\nabla F(M) =\frac{1}{m} \A^*\A(M)$ is concentrated around its mean with high probability.

\begin{lemma}[Concentration of $\frac{1}{m}\A^*\A(M)$]\label{GrMainCon}
Let $M\in\R^{p\times p}$ be a fixed matrix with rank $r$ and let $S_i = x_ix_i^T(M)x_ix_i^T$ for $i=1,\ldots,m$. Consider the linear operator $\A$ in model~\eqref{Obmodel} independent of $M$. 
Then with probability at least $1-\xi_2$, we have:
\begin{align}
\Big{\|}\frac{1}{m}\sum_{i=1}^mS_i - \E S_i\Big{\|}_2\leq C'\sqrt{\frac{pr^2\log^3p}{m}\log(\frac{p}{\xi_2})}\|M\|_2.
\end{align}
where $C'>0$ is a constant. 
\end{lemma}
\begin{proof}
In all the following expressions, $C_l>0$ for $l=1,\ldots,11$ are absolute constants. First we note that  by some calculations, one can show that 
$$\E\left(\frac{1}{m}\A^*\A(M)\right) = \E S_i = 2(M) + Tr(M)I.$$
Our technique to establish the concentration of $ \A^*\A(L_t-L_*)$ is based on the matrix Bernstein inequality. As stated in lemma~\eqref{bernMat}, there should be a spectral bound on the summands, $S_i = x_ix_i^T(M)x_ix_i^T$ for $i=1,\ldots,m$. Since the entries of $a_i$ are Gaussian, we cannot use directly matrix Bernstein inequality. Inspiring by~\cite{zhong2015efficient}, we will use a \emph{truncation} trick to make sure that the spectral norm of summands are bounded. 
Define the random variable $\T{x_i}^{(j)}$ as follows:
\begin{align}\label{Trunc}
\T{x_i}^{(j)}= \begin{cases}
x_i^{(j)}, \qquad |x_i^{(j)}|\leq C_1\sqrt{\log mp} \\
0,\qquad\quad otherwise,
\end{cases}
\end{align}
where $x_i^{(j)}$ is the $j^{th}$ entry of the random vector $x_i$. By this definition, we have immediately the following properties:
\begin{itemize}
\item $\pr\left(x_i^{(j)} = \T{x_i}^{(j)}\right)\geq 1-\frac{1}{\left(mp\right)^{C_2}}$, 
\item $\E\left(\T{x_i}^{(j)}\T{x_i}^{(k)}\right) = 0$, for $j\neq k$,
\item $\E\T{x_i}^{(j)} = 0$ for $j=1,\ldots, p$,
\item $\E\left(\T{x_i}^{(j)}\right)^2\leq\E\left(x_i^{(j)}\right)^2 =1$, \ for $j=1,\ldots, p$,
\end{itemize}
Let $\T{S_i} =\T{x_i}\T{x_i}^TM\T{x_i}\T{x_i}^T$ for $i=1,\ldots,m$. 
We need to bound parameters $R$ and $\sigma$ in matrix Bernstein inequality. Denote the SVD of $M$ by $M= U_M\Sigma V_M^T$. Since $x_i$ is a normal random vector, it is rotationally invariant. As a result, w.l.o.g., we can assume that $U_M = [e_1,e_2,\ldots,e_{r}]$ and $V_M = [e_1,e_2,\ldots,e_{r}]$ as long as the random vector $x_i$ is independent of $M$. Here, $e_j$ denotes the $j^{th}$ canonical basis vector in $\R^p$. To make sure this happens, we use $m$ fresh samples of $x_i$'s in each iteration of the algorithm. 

Now, we have for each $i$:
\begin{align*}
\|\T{x_i}\T{x_i}^TM\T{x_i}\T{x_i}^T\|_2 &= \|\T{x_i}\T{x_i}^TU_M\Sigma V_M^T\T{x_i}\T{x_i}^T\|_2 \\
&\leq |\T{x_i}^TU_M\Sigma V_M^T\T{x_i}|\|\T{x_i}\T{x_i}^T\|_2\\
&\leq \|U_M^T\T{x_i}\|_2\|V_M^T\T{x_i}\|_2\|\T{x_i}\|_2^2\|M\|_2\\
&\overset{a_1}{\leq}pr\|\T{x_i}\|_{\infty}^4\|M\|_2\\
&\overset{a_2}{\leq}C_3pr\log^2(mp)\|M\|_2,
\end{align*}
where $a_1$ holds due to rotational variant discussed above and relation between $\ell_2$ and $\ell_{\infty}$. Also, $a_2$ is due to applying bound in~\eqref{Trunc}. Now, we can calculate $R$:
$$\|\T{S_i}-\E \T{S_i}\|_2\leq\|\T{S_i}\|_2 + \E\|\T{S_i}\|_2\leq 2\|\T{S_i}\|_2\leq C_4pr\log^2(mp)\|M\|_2 = R,$$
where we have used both the triangle inequality and Jensen's inequality in the first inequality above. For $\sigma$, we define $\T{S}$ as the truncated version of $S$, independent copy of $S_i$'s. Hence:
\begin{align*}
\sigma &= \big{\|}\E\T{ S }^2 - (\E \T{S})^2\big{\|}_2\\
& \overset{a_1}{\leq}\|\E\T{S}^2\|_2= \Big{\|}\E\left( \T{x}\T{x}^TM\T{x}\T{x}^T\T{x}\T{x}^TM\T{x}\T{x}^T\right) \Big{\|}_2 \\
&=\Big{\|}\E\left(\|\T{x}\|_2^2\left(\T{x}^TM\T{x}\right)^2\T{x}\T{x}^T\right) \Big{\|}_2\\
&\overset{a_1}{\leq} C_5pr^2\log^3(pm)\|M\|_2^2\Big{\|}\E\left(\T{x}\T{x}^T\right) \Big{\|}_2\\
&\overset{a_2}{\leq}C_5pr^2\log^3(pm)\|M\|_2^2,
\end{align*}
where $a_1$ is followed by positive semidefinite of $(\E\T{S})^2$. In addition, $a_2$ holds due to the upper bound on $\left(\T{x}^TM\T{x}\right)^2\|\T{x}\|_2^2$:
\begin{align*}
\left(\T{x}^TM\T{x}\right)^2\|\T{x}\|_2^2 &= \left(\T{x}^T U_M\Sigma V_M^T\T{x}\right)^2\|\T{x}\|_2^2\\
&\leq \|U_M^T\T{x}\|_2^2\|V_M^T\T{x}\|_2^2\|M\|_2^2\|\T{x}\|_2^2\\
&\leq pr^2\|\T{x}\|_{\infty}^6\|M\|_2\\
&\leq C_6pr^2\ \log^3(mp)\|M\|_2,
\end{align*}
where we have again used the same argument of rotational invariance. Finally, $a_2$ holds due to the fact that $\E\left(\T{x_i}\T{x_i}^T\right)\preceq I$. Now, we can use the matrix Bernstein inequality for bounding $\Big{\|}\frac{1}{m}\sum_{i=1}^m\T{S_i} - \E\T{S_i}\Big{\|}_2$: 

\begin{align}
\pr\left(\Big{\|}\frac{1}{m}\sum_{i=1}^m\T{S_i} - \E\T{S_i}\Big{\|}_2\geq t\right)&\leq 2p\exp\left(\frac{-mt^2}{\sigma +Rt/3}\right) \nonumber \\
&\leq2p\exp\left(\frac{-mt^2}{C_5pr^2\log^3(pm)\|M\|_2^2 +C_4pr\log^2(mp)\|M\|_2t/3}\right) \nonumber \\
&\overset{a_1}{\leq}2p\exp\left(\frac{-mt^2}{C_7pr^2\log^3(pm)\|M\|_2^2} \right),
\end{align}
where $a_1$ holds by choosing constant $C_7$ sufficiently large. Now choose $t\geq\|M\|_2\sqrt{C_8\frac{pr^2\log^3(pm)}{m}\log(\frac{p}{\xi'_2})}$. Thus with probability at least $1-\xi'_2$, we have:
\begin{align*}
\Big{\|}\frac{1}{m}\sum_{i=1}^m\T{S_i} - \E\T{S_i}\Big{\|}_2\leq \sqrt{C_8\frac{pr^2\log^3(pm)}{m}\log(\frac{p}{\xi'_2})}\|M\|_2,
\end{align*}
This bound shows that by taking $m=\O(\frac{1}{\theta^2}pr^2\log^3p\log(\frac{p}{\xi'_2}))$ for some $\theta>0$, we can bound the multiplicative term by $\theta$. Actually, this choice of $m$ determines the sample complexity of EP-ROM and we will return back to this issue in subsequent sections. Recall that $\T{S_i}$ includes the truncated random variables, i.e, $\T{S_i} =\T{x_i}\T{x_i}^TM\T{x_i}\T{x_i}^T$. Also, $\pr\left(x_i^{(j)} = \T{x_i}^{(j)}\right)\geq 1-\frac{1}{\left(mp\right)^{C_2}}\geq 1-\frac{1}{\left(p\right)^{C_9}}$. Hence, we need to extend our result to the original $x_i$. By definition of $\T{x_i}$ in~\eqref{Trunc} and choosing constant $C_9$ sufficiently large ($C_9>1$),  we have $\pr\left(\|S_i - \T{S_i}\|_2=0\right) = \pr\left(\|x_ix_i^T -\T{x_i}\T{x_i}^T\|_2\right)\geq 1-\frac{1}{(p)^{C_{10}}}$. Here we have used the union bound over $p^2$ variables. Since we have $m$ random matrices $S_i$, we need to take another union bound. As a result with probability $1 - \xi_2$ where $\xi_2 = \frac{1}{(p)^{C_{11}}}$, we have:
\begin{align}
\Big{\|}\frac{1}{m}\sum_{i=1}^mS_i - \E S_i\Big{\|}_2\leq\sqrt{C_8\frac{pr^2\log^3p}{m}\log(\frac{p}{\xi_2})}\|M\|_2.
\end{align}
\end{proof}

\begin{proof}[Proof of theorem~\ref{CURIP}]
\begin{align}
&\big{\|}L_1 - L_2- \frac{\rho_1}{m}\A^*\A(L_1-L_2) + \rho_1Tr(L_1-\bar{y})I \big{\|}_2 \nonumber\\
&\hspace{24mm}\overset{a_1}{\leq}\|L_1 - L_2- \frac{\rho_1}{m}\A^*\A(L_1-L_2) + \rho_1 Tr(L_1-L_2)I \|_2 + \rho_1\|\left(\bar{y}-Tr(L_2)I\right)\|_2\nonumber \\
&\hspace{24mm}\overset{a_2}{\leq}\|\frac{\rho_1}{m}\A^*\A(L_1-L_2) -\rho_1\left(2(L_1 - L_2)+ Tr(L_1-L_2)I\right)\|_2  \nonumber \\
&\hspace{24mm}\hspace{54mm} + (2\rho_1-1)\|L_1 - L_2\|_2+ \rho_1\|\left(\bar{y}-Tr(L_2)I\right)\|_2 \nonumber \\
&\hspace{24mm}\overset{a_3}{\leq}\left(\rho_1 C'\sqrt{\frac{pr^2\log^3p}{m}\log(\frac{p}{\xi_2})} + (2\rho_1-1)\right)\|L_1 - L_2\|_2+ \rho_1 C\sqrt{\frac{1}{m}\log(\frac{p}{\xi_1})}Tr(L_2)  \nonumber \\
&\hspace{24mm}\overset{a_4}{\leq}(\rho_1\delta_1 - 1)\|L_1-L_2\|_2  + \rho_1\delta_2Tr(L_2),  \nonumber
\end{align}
where $a_1$ is followed by adding and subtracting of $Tr(L_*)I$,
inequality $a_2$ follows from triangle inequality, $a_3$ holds with probability $1-\xi_1-\xi_2 = 1-\xi$ by invoking Lemma~\ref{TrMainCon}, and Lemma~\ref{GrMainCon} (by fixed matrix $L_t-L_*$ with rank $2r$), and finally $a_4$ is followed by choosing $m=\O\left(\frac{1}{\delta^2}\max\left(pr^2\log^3p\log(\frac{p}{\xi}),\ r^2\log(\frac{p}{\xi})\|L_2\|_2^2\right)\right)$ for some $\delta>0$. Here, $\delta_1>0$ and $\delta_2>0$ are arbitrary small constants by the choice of $m$ as mentioned. Finally, by forcing $0<\rho_1\delta_1 - 1<1$, the proof is completed.
\end{proof}

\begin{corollary}\label{CURIPConseq}
From Theorem~\ref{CURIP} we have the following conclusions:
\begin{enumerate}
\item Let $U$ be the bases for the column space of fixed matrices $L_1$ and $L_2$ such that $\rank(L_i)\leq r$ for $i=1,2$ and $\P_U$ is the projection onto it. Also consider all the assumptions of Theorem~\ref{CURIP}. Then 
$$\big{\|}L_1 - L_2- \frac{\rho_1}{m}\P_U\A^*\A(L_1-L_2) + \rho_1 \P_UTr(L_1-\bar{y})I \big{\|}_2\leq(\rho_1\delta_1 - 1)\|L_1-L_2\|_2  + \rho_1\delta_2Tr(L_2),$$
\item $\big{\|}\frac{\rho_1}{m}\A^*\A(L_1-L_2) - \rho_1Tr(L_1-\bar{y})I \big{\|}_2\geq (2-\rho_1\delta_1)\|L_1-L_2\|_2 -  \rho_1\delta_2Tr(L_2)$,
\end{enumerate}
\end{corollary}

\begin{proof}
The first result holds by the fact that $L_1-L_2$ lies in subspace $U$. The second result is directly followed by Theorem~\ref{CURIP}.
\end{proof}

\begin{proof}[Proof of Theorem~\ref{Staterror}]
The proof is very similar to the proof of lemma~\ref{GrMainCon} and we only give a brief sketch. The idea is again to use the matrix Bernstein inequality; to do this, we have to use the truncation trick both on the random vector $x_i$ and the noise vector $e$. We introduce $\T{x_i}$ as~\eqref{Trunc} and similarly $\T{e}$ as follows ($j=1,\ldots,m$):
\begin{align}
\T{e}^{(j)}= \begin{cases}
e^{(j)}, \qquad |e^{(j)}|\leq c'_1\sqrt{\log m} \\
0,\qquad\quad otherwise,
\end{cases}
\end{align} 
In the following expressions, $c'_l>0$ for $l=1,4$ are absolute constants and $c'_l>0$ for $l=2,3,5,6$ are some constants depend on $\tau$. Let $W_i = \T{e_i}\T{x_i}\T{x_i}^T$ for $i=1,\ldots,m$ and $W = \T{e_r}\T{x}\T{x}^T$ be a independent copy of $W_i$'s (i.e, $\T{e_r}$ and $\T{x}$ are independent copies of $e_i$ and $x_i$, respectively). Hence, $\E\frac{\A^*e}{m} = \frac{1}{m}\sum_{i=1}^{m}\E\T{e_i}\T{x_i}\T{x_i}^T = \E S_i = 0$ and $\pr(\T{e_i} = e_i)\geq 1-\frac{1}{m^{c'_2}}$ by assumptions on $e$. Now parameters $R$ and $\sigma$ in matrix Bernstein inequality can be calculated as follows:
\begin{align*}
\sigma = \|\E WW^T\|_2 = \Big{\|}\E \T{e_r}^2\E(\|\T{x}\|_2^2\T{x}\T{x}^T)\Big{\|}_2\leq c'_3p\log(mp),\\
R = \|\T{e_r}\T{x}\T{x}^T\|_2\leq c'_4p\sqrt{\log m}\log(mp),
\end{align*}
As a result for all $t_3\geq 0$, we have 
\begin{align*}
\pr\left(\Big{\|}\frac{1}{m}\sum_{i=1}^{m}W_i \Big{\|}_2\geq t_3\right)\leq 2p\exp\left(\frac{mt_3^2}{\sigma + Rt_3/3}\right)\leq 2p\exp\left(\frac{mt_3^2}{c'_5p\log(mp)}\right),
\end{align*}
where the last inequality holds by sufficiently large $c'_5$. Now similar to lemma~\ref{GrMainCon} by choosing $t_3\geq\sqrt{C''_{6}\frac{p\log^2p}{m}\log(\frac{p}{\xi_3})}$ and union bound, we obtain with probability at least $1-\xi_3$:
$$\Big{\|}\frac{1}{m}\A^*e\Big{\|}_2\leq\sqrt{C''_6\frac{p\log p}{m}\log(\frac{p}{\xi_3})}.$$
\end{proof}

\subsection{Running time analysis}\label{TIMEAnalysis}
\textbf{Running time of EP-ROM.} Each iteration of EP-ROM involves evaluation of the gradient at current estimation and an exact projection on the set of rank $r$ matrices. Recall that unbiased gradient of the objective function is given by:
$$\nabla F(L_t) +\left(Tr(L_t)-\bar{y}\right)I = \frac{1}{m}\sum_{i=1}^{m}\left( x_i^TL_tx_i-y_i\right) x_ix_i^T + \left(Tr(L_t)-\bar{y}\right)I.$$
The inner term $\left( x_i^TL_tx_i-y_i\right)$ can be computed only once per iteration and stored in a temporary vector $d\in\R^m$. Since in each iteration, we have access to the factors of $L_t = U_tV_t^T$ such that $U_t,V_t\in\R^{p\times r}$, the calculation of $d$ takes $\O(pr)$ operations. Then we can calculate $dx_ix_i^T$ in $\O(p^2)$ operations. As a result, calculating the whole unbiased gradient takes $\O(mp^2)$ which implies ${\O}(p^3r^2\log^4(p)\log(\frac{1}{\epsilon}))$ due to the choice of $m$. On the other hand, exact projection on the set of rank $r$ matrices takes $\O(p^3)$, since the SVD of even a rank-1 $p\times p$ matrix (without spectral assumptions) needs $\O(p^3)$ operations. As a result, the total running time for EP-ROM to achieve $\epsilon$ accuracy is given by $K = {\O}(p^3r^2\log^4(p)\log^2(\frac{1}{\epsilon}))$ due to the linear convergence of EP-ROM. 

We note that even if we use Matlab's \textsc{SVDS}, which uses the Lanczos method for the projection step, the required running time equals $\widetilde{\O}(\frac{p^2r}{\sqrt{\delta-1}})$ where $\delta$ denotes the gap between the $r^{th}$ and $(r+1)^{th}$ largest singular values. Hence, the gradient calculation is the computationally dominating step and the total running time is as before.

As discussed before, we use MBK-SVD as head approximate projection in AP-ROM. The pseudocode for MBK-SVD is given in Algorithm~\ref{alg:MBKSVD}. 

\begin{algorithm}
\caption{MBK-SVD}
\label{alg:MBKSVD}
\begin{algorithmic}
\STATE \textbf{Inputs:} $y$, measurement operator, $\A =\{x_1x_1^T,x_2x_2^T\ldots,x_mx_m^T\}$, rank $r$, block size $b = r+5$, $\vartheta\in(0,1)$
\STATE \textbf{Outputs:} matrix $Z\in\R^{p\times r}$ 
\STATE 1: Set $q = \Theta(\frac{\log p}{\sqrt{\vartheta}})$ and $G\thicksim\mathcal{N}(0,1)^{p\times b}$
\STATE 2: Calculate $\bar{y} = \frac{1}{m}\sum_{i=1}^{m}y_i$ and $d=  x_i^TL_tx_i-y_i$
\STATE 3: Allocate Krylov subspace, $Kr\in\R^{p\times q}$.
\STATE 4: $I\leftarrow \mathcal{B}(\A,G,d,\bar{y})$, $G\leftarrow I$, $Kr[:,1:b]\leftarrow I$
\FOR{$i=2:q$}
\STATE $I\leftarrow \mathcal{B}(\A,G,d,\bar{y})$
\STATE $J\leftarrow \mathcal{B}(\A,I,d,\bar{y})$
\STATE $Kr[:,(i-1)b+1:ib]\leftarrow J$
\STATE $G\leftarrow J$
\ENDFOR
\STATE 5: Orthonormalize the columns of Kr to find $Q\in\R^{p\times qb}$.
\STATE 6: Compute $M\leftarrow\mathcal{B}(\A,Q,d,\bar{y})$, $M\leftarrow M^T$
\STATE 7: Compute top $r$ singular vectors of $M$ and call it $\overline{U_k}$.
\STATE \textbf{Return}: $Z=Q\overline{U_k}$ 
\end{algorithmic}
\end{algorithm}

\begin{algorithm}
\caption{Operator $\mathcal{B}(\A,G,d,\bar{y})$}
\label{alg:operatorB} 
\begin{algorithmic} 
\STATE \textbf{Inputs:} $\A,G,d,\bar{y}$
\STATE \textbf{Outputs:} $W_3 = \left(\frac{1}{m}\sum_{i=1}^{m}\left( x_i^TL_tx_i-y_i\right) x_ix_i^T-\left(Tr(L_t)-\bar{y}\right)I\right)G$
\FOR{$j=1:m$}
\STATE $W_1\leftarrow x_j^TG$
\STATE $W_2\leftarrow d^{(i)}x_jW_1$
\STATE $W_3\leftarrow \frac{W_2}{m} - \left(Tr(L_t)-\bar{y}\right)G$
\ENDFOR
\STATE\textbf{Return:} $W_3$.
\end{algorithmic}
\end{algorithm}

In Algorithm~\ref{alg:MBKSVD}, $Kr$ denotes a Krylov subspace, and the parameter $b$ determines the size of each block inside $\textbf{kr}$ which can be any value greater than $r$. Also, $\vartheta$ represents the desired accuracy in calculating of the projection. 

Now let $\Delta=\frac{1}{m}\sum_{i=1}^{m}\left( x_i^TL_tx_i-y_i\right) x_ix_i^T-\left(Tr(L_t)-\bar{y}\right)I$. In MBK-SVD, the computation of vector $d$ takes $\O(pr)$ operations as before. In addition, instead of multiplying unbiased gradient term by a random matrix, each sensing vector, $x_i$ is multiplied by a matrix $G$ which needs $\O(pr)$ operations. To be more precise, the Krylov subspace is formed by $q$ iterations. Each iteration needs to compute the product of $(\Delta^2)^k.\Delta.G$ for $k=0,\ldots,q$ and this is done through operator $\mathcal{B}$. The code for this operator is given in Algorithm~\ref{alg:operatorB}. To run this algorithm, we need $\O(mpr)$ operations; there are $m$ iterations and each of them takes $\widetilde{\O}(pr^2\log(\frac{1}{\epsilon}))$ time ($\widetilde{\O}$ hides dependency on $polylog(p)$). As a result, MBK-SVD requires $\O(qmpr)$ operations which implies that the total running time of MBK-SVD is scaled as ${\O}\left(\max\left(p^2\log^4(p), p\log(p)\|L_*\|_2^2\right)r^3\log(p)\log(\frac{1}{\epsilon})\right)$ by the choice of $m$ and $q$. 

\begin{proof}[Proof of Theorem~\ref{runningtimeAP}]
As we discussed before, AP-ROM uses two tail and head approximate projections. For implementing head, we use MBK-SVD with rank set to $2r$ to obtain the approximation of right singular vectors. Let $U_{\H}$ be the returned $2r$-dimensional subspace by MBK-SVD. Now we have to form $U_tV_t^T - \eta U_{\H}U_{\H}^T\Delta$ which is a matrix with rank at most $3r$. Here, $U_t,V_t$ are factors of $L_t$. To efficiently compute this expression, we again use operator $\mathcal{B}$ by calculating $U_{\H}^T\Delta = \left(\mathcal{B}(\A,U_{\H},d,\bar{y})\right)^T$ in $\O(pr)$ operations. Now to apply the approximate tail projection, we can use either the Lanczos algorithm (SVDs) or ordinary BK-SVD, both of which require $\O(pr^2)$ operations. After calculating the $r$-dimensional subspace returned by tail operator, $U_{\Ta}$ we can project $U_tV_t^T - \eta U_{\H}U_{\H}^T\Delta$ onto it which needs another $\O(pr^2)$ operations. As a result, the total running time for AP-ROM to achieve $\epsilon$ accuracy is scaled as $K = {\O}\left(\max\left(p^2\log^4(p), p\log(p)\|L_*\|_2^2\right)r^3\log(p)\log^2(\frac{1}{\epsilon})\right)$ due to the linear convergence of EP-ROM. 
\end{proof}

\end{document}